\documentclass[journal]{IEEEtran}

\usepackage[colorlinks,urlcolor=black,linkcolor=black,citecolor=black]{hyperref}

\makeatletter
\renewcommand{\p@subsection}{\thesection-}
\renewcommand{\p@subsubsection}{\thesection-\thesubsection}
\makeatother

\usepackage{graphicx}
\usepackage{cite}

\usepackage{amsmath,amssymb,amsfonts,bm}
\usepackage{mathtools}
\usepackage{amsthm}

\usepackage[table,x11names]{xcolor}
\usepackage{booktabs}
\usepackage{array}
\usepackage{multirow}
\usepackage{makecell}
\usepackage{tabularx}
\usepackage{siunitx}
\usepackage{bbm}
\usepackage{algorithm}
\usepackage{algorithmic}
\usepackage{placeins}
\usepackage{caption}    % Needed for captioning images
\usepackage{subcaption} %subfigure
\usepackage[acronyms]{glossaries}

\usepackage{tikz}
\usepackage{pgfplots}
\usepackage{pgfplotstable}
\pgfplotsset{compat=1.18}

\usetikzlibrary{backgrounds}
\usepgfplotslibrary{fillbetween}
\newcounter{todocounter}

\NewDocumentCommand{\todo}{s m}{%
  \IfBooleanTF{#1}
    {\textcolor{green}{To do: #2}}%
    {\refstepcounter{todocounter}\textcolor{green}{To do \thetodocounter: #2 \newline}}%
}

\newcounter{examinecounter}

\NewDocumentCommand{\examine}{s m}{%
  \IfBooleanTF{#1}
    {\textcolor{green}{Examine: #2}}%
    {\refstepcounter{examinecounter}\textcolor{blue}{Examine \theexaminecounter: #2 \newline}}%
}

\newcounter{Questionscounter}

\NewDocumentCommand{\ques}{s m}{%
  \IfBooleanTF{#1}
    {\textcolor{green}{Questions: #2}}%
    {\refstepcounter{Questionscounter}\textcolor{red}{Questions \theQuestionscounter: #2 \newline}}%
}

\newcommand{\figref}[1]{\hyperref[#1]{Fig.~\ref*{#1}}}
\newcommand{\tabref}[1]{\hyperref[#1]{Tab.~\ref*{#1}}}
\newcommand{\secref}[1]{\hyperref[#1]{Sec.~\ref*{#1}}}

\newcommand{\appref}[1]{\hyperref[#1]{Appendix~\ref*{#1}}}

\newacronym{sos}{SOS}{sum-of-squares}
\newacronym{clf}{CLF}{control Lyapunov function}
\newacronym{cbf}{CBF}{control barrier function}
\newacronym{qp}{QP}{quadratic program}
\newacronym{zoh}{ZOH}{zero-order-hold}
\newacronym{mpc}{MPC}{model predictive controller}
\newacronym{sdp}{SDP}{semidefinite program}
\newacronym{lmis}{LMIs}{linear matrix inequalities}
\newacronym{psd}{PSD}{positive-semidefinite}
\newacronym{bbs}{JBS}{joint barrier synthesis}
\newacronym{sbs}{SBS}{successive-barrier synthesis}

\newacronym{pwbb}{JBS-PW}
{joint barrier synthesis with piecewise-polynomial dynamics}

\newacronym{ldbb}{JBS-LD}
{joint barrier synthesis with lifted dynamics}

\newacronym{pwsb}{SBS-PW}
{successive-barrier synthesis with piecewise-polynomial dynamics}

\newacronym{ldsb}{SBS-LD}
{successive-barrier synthesis with lifted dynamics}

\newacronym{CT}{CT}{coordinated turn}
\newacronym{PM}{PM}{planar multirotor}

\newacronym{smt}{SMT}{satisfiability modulo theories}

\theoremstyle{plain}
\newtheorem{theorem}{Theorem}[section]
\newtheorem{proposition}{Proposition}[section]

\theoremstyle{definition}
\newtheorem{definition}{Definition}[section]

\theoremstyle{remark}
\newtheorem{remark}{Remark}[section]

\begin{document}

% \sptitle{Article Category}
% \sptitle{Regular Paper}

\title{Barrier Certificate Synthesis for Non-Polynomial Robotic Dynamics via Polynomial Lifting}

% \editor{This paper was recommended by Associate Editor F. A. Author.}

% \author{F. A. AUTHOR\affilmark{1} (Student Member, IEEE)}

% \author{B. AUTHOR\affilmark{2}  (Member, IEEE)}

% \author{C.~AUTHOR\affilmark{2}}

% \affil{National Institute of Standards and Technology, Boulder, CO 80305 USA} 
% \affil{Department of Physics, Colorado State University, Fort Collins, CO 80523 USA} 

% \corresp{CORRESPONDING AUTHOR: F. A. Author (e-mail: \href{mailto:author@boulder.nist.gov}{author@boulder.nist.gov})}
% \authornote{This work was supported by the Canada Research Chair CRC and the National Research and Engineering Council of Canada
% NSERC.}

% \markboth{PREPARATION OF PAPERS FOR IEEE OPEN JOURNAL OF CONTROL SYSTEMS}{F. A. AUTHOR {\itshape ET AL}.}
% \editor{This paper was recommended by Associate Editor F. A. Author.}

% \author{Shivam Chaubey\affilmark{1}  (Student Member, IEEE)}
% \author{Francesco Verdoja\affilmark{1} (Senior Member, IEEE)}
% \author{Shankar Deka\affilmark{1}(Member, IEEE)}
% \author{Ville Kyrki\affilmark{1} (Senior Member, IEEE)}

% \affil{School of Electrical Engineering, Aalto University, Espoo, Finland}

% \corresp{CORRESPONDING AUTHOR: Shivam Chaubey
% (e-mail: \href{mailto:shivam.chaubey@aalto.fi}
% {shivam.chaubey@aalto.fi})}
% \authornote{The authors acknowledge the use of the MIDAS infrastructure of Aalto School of Electrical Engineering.}
\author{
Shivam Chaubey,
Francesco Verdoja,
Shankar Deka,
and Ville Kyrki
\thanks{The authors are with the School of Electrical Engineering, Aalto University, Espoo, Finland, and acknowledge the use of the university's MIDAS infrastructure.}
\thanks{This work has been submitted to the IEEE for possible publication.
Copyright may be transferred without notice, after which this version
may no longer be accessible.}%
}

\maketitle

\begin{abstract}
Safe operation of robotic systems requires trajectories to remain within a
prescribed safe set under admissible control inputs. Barrier certificates
provide such guarantees by certifying a controlled-invariant region within
that set.
Sum-of-squares 
optimization offers a systematic way to synthesize such certificates, 
but its direct application requires polynomial dynamics, excluding 
common robotic nonlinearities, including trigonometric terms. We address 
this limitation using exact polynomial lifting, which replaces non-polynomial dynamics with polynomial-augmented dynamics subject to 
lifting-induced algebraic constraints, preserving nonlinear geometry without approximation.
We formulate lifted-domain joint barrier synthesis that computes 
a certificate with a state-feedback control witness and develop a 
sampled-data safety filter for zero-order-hold implementation. To 
assess whether the benefits of lifting persist across synthesis 
frameworks, we also adapt a
sample-guided successive-barrier 
method to the lifted representation. On coordinated-turn and planar 
multirotor models, exact lifting improves certified coverage in both 
methods: at matched sample sizes, lifted successive-barrier synthesis achieves higher coverage with fewer barriers and 
lower computational cost, while lifted joint barrier synthesis provides higher 
coverage and lower computational cost than the finest tested piecewise 
resolution. In closed-loop experiments, the safety filter 
maintains feasibility and safety across all evaluated
trajectories, reduces spatial conservativeness, and requires less 
intervention for both models.
\end{abstract}

% \begin{IEEEkeywords}
% Enter key words or phrases in alphabetical order, separated by commas. For a list of\goodbreak suggested keywords, send a blank e-mail to \href{mailto:keywords@ieee.org}{keywords@ieee.org} or visit\goodbreak \href{http://www.ieee.org/organizations/pubs/ani_prod/keywrd98.txt}{http://www.ieee.org/organizations/pubs/ani\_prod/keywrd98.txt}
% \end{IEEEkeywords}
\begin{IEEEkeywords}
Formal verification/synthesis, nonlinear systems
and control, sampled-data control, robotics, computational
methods.
% https://ojcsys.github.io/pages/keywords.html
\end{IEEEkeywords}

% https://www.sciencedirect.com/journal/systems-and-control-letters
\section{Introduction}\label{sec:intro}

Ensuring safety in autonomous robotic systems requires
controllers that can formally guarantee constraint satisfaction
at all times. \Glspl{cbf} have emerged as a principled
framework for enforcing forward invariance of safe sets, while
remaining compatible with performance-oriented control
objectives~\cite{ames2016control,8796030}.
\glspl{cbf} have been applied across robotic systems, including
multi-robot collision 
avoidance~\cite{7857061}, bipedal
locomotion~\cite{7172044}, robotic
grasping~\cite{8913716}, and vehicle
safety~\cite{ames2016control}.
They are commonly deployed through a real-time \gls{qp} that
minimally modifies a nominal control input, making them suitable
for safety filtering applications such as shared
autonomy, teleoperation, and supervisory control
~\cite{losey2021learninglatentactionscontrol,10015627}.

A \gls{cbf}-based safety filter requires a certified region that
satisfies the state constraints and is controlled invariant under
admissible inputs. In practice, \glspl{cbf} are often derived in a
task- or system-specific manner; while such constructions can be
rigorously validated, they are typically tailored to a particular
system or safety specification
~\cite{dai2022learning,DuanControl}.
Systematically synthesizing such certificates while accounting for
the system dynamics, state constraints, and admissible inputs remains
computationally challenging~\cite{9683520,10015199}.

For systems with polynomial dynamics, \gls{sos} programming
offers a systematic approach: the barrier-certificate
conditions are encoded as polynomial nonnegativity constraints
and solved through semidefinite programming
\cite{prajna_hybrid,4287147,8431617,
10176323,toulkani2025minimally}.
However, direct \gls{sos} formulations require polynomial
system dynamics. This limits their direct application to many
robotic systems, where trigonometric nonlinearities arising
from rotational kinematics and dynamics are ubiquitous.
Moreover, rotational states evolve on non-Euclidean
configuration spaces, so representing them directly in
Euclidean coordinates can introduce artificial coordinate
boundaries and periodicity issues.
Such dynamics can be made amenable to \gls{sos} synthesis
through certified polynomial approximation
\cite{YANG2020100837,Wajid_successive}, or through
exact polynomial recasting using auxiliary states and
lifting-induced algebraic constraints
\cite{Papachristodoulou2005}. 
These representations introduce
different structural trade-offs: piecewise-polynomial
approximation retains the original state dimension but
requires certification across local regions and
approximation-error uncertainty, whereas exact lifting
avoids approximation-error uncertainty at the cost of
additional polynomial states and algebraic constraints.

A second challenge arises when a certificate synthesized for
continuous-time dynamics is deployed through a sampled-data
\gls{cbf}-\gls{qp} safety filter. 
% Many \gls{cbf} formulations prescribe a positive barrier rate during
% synthesis, which can restrict the synthesized certificate.
Many \gls{cbf} formulations prescribe a positive barrier rate during
synthesis, which can make the forward-invariance condition more
restrictive.
Separately, the runtime barrier rate may also be adjusted
online to improve filter feasibility or reduce conservativeness
~\cite{zeng2021optimal,garg2024advances}.
% Many \gls{cbf} formulations
% prescribe a positive barrier rate during synthesis, which can
% restrict the certificate, while the runtime rate also affects
% filter feasibility and conservativeness
% ~\cite{zeng2021optimal,garg2024advances}.
% In boundary-based synthesis~\cite{Wang2023}, invariance is
% imposed only on \(B(\mathbf{x})=0\), following Nagumo's
% boundary condition~\cite{Nagumo1942berDL}.
Boundary-based synthesis~\cite{Wang2023} avoids 
% prescribing such
% a rate 
prescribing a barrier rate during synthesis
by imposing invariance only on \(B(\mathbf{x})=0\), following
Nagumo's boundary condition~\cite{Nagumo1942berDL}.
This boundary-only condition allows a
state-dependent polynomial multiplier used in the forward-invariance condition to remain unrestricted in sign
during synthesis.
This provides greater synthesis freedom than
prescribing a fixed positive barrier rate, but
directly using this multiplier as the runtime barrier rate can be unnecessarily restrictive at states where it is negative.
% the resulting state-dependent multiplier does not directly provide a suitable
% fixed rate for online filtering and can be unnecessarily
% restrictive when it becomes negative.
Moreover, enforcing the
continuous-time barrier condition only at sampling instants
does not guarantee safety between samples under \gls{zoh}
control~\cite{9417092,9993226,
liu2025samplingawarecontrolbarrierfunctions}. Sampled-data
deployment therefore requires both a synthesis-compatible
runtime barrier rate and an inter-sample safety condition.

Two gaps therefore remain: (i) the effect of exact lifting versus
certified piecewise-polynomial representations under matched
barrier-synthesis methods, and (ii) a fixed-rate bridge from
boundary-based synthesis to inter-sample-safe \gls{zoh} deployment.
We examine the first under \gls{bbs}~\cite{Wang2023}, which jointly
synthesizes a barrier and an admissible state-feedback control witness, and
\gls{sbs}~\cite{Wajid_successive}, which uses samples to
successively construct barriers for fixed admissible controls.
For the second, we
construct a safety filter from the \gls{bbs} certificate.
% we
% address the second for \gls{bbs}.

The main contributions of this work are as follows:
\begin{enumerate}
\item We formulate lifted-domain \gls{sos} barrier-certificate synthesis for control-affine robotic systems by exactly representing their non-polynomial dynamics as polynomial dynamics in a lifted state space and enforcing the associated algebraic constraints.

\item Using this lifted-domain formulation, we develop~\gls{ldbb} and adapt~\gls{sbs} to obtain \gls{ldsb}. We compare them with their matched certified piecewise-polynomial counterparts, \gls{pwbb} and \gls{pwsb}, respectively, to evaluate the effect of the dynamics representation across both synthesis methods.

\item For \gls{bbs}, we derive a fixed runtime barrier-rate
interval: an \gls{sos}-based upper bound on the free synthesis
multiplier gives the lower rate bound, while
\(\gamma T_s\le 1\) gives the upper bound. We then incorporate
the local second-order inter-sample \gls{cbf} margin
of~\cite{9417092} into the sampled-data
\gls{cbf}-\gls{qp} safety filter.
\end{enumerate}

The remainder of this paper reviews related work and
preliminaries in~\secref{sec:related_work}--\ref{sec:problem},
presents the~\gls{bbs}, sampled-data safety-filter, and~\gls{sbs} formulations in~\secref{sec:bilinear_method}--\ref{sec:successive_barrier_synthesis},
and reports the evaluation, discussion, and conclusions in~\secref{sec:experiments}--\ref{sec:conclusion}.
Additional derivations and implementation details are provided
in the supplementary material~\footnote{\url{https://version.aalto.fi/gitlab/irobotics/barrier_certificate_synthesis_for_non_polynomial_robotic_dynamics.git}} and Appendix.
\section{Related Work}\label{sec:related_work}

For polynomial systems, \gls{sos}-based barrier-certificate
synthesis under state and input constraints has been extensively
studied~\cite{8431617,10176323,toulkani2025minimally,Wang2023}.
For non-polynomial dynamics, one approach approximates the
system dynamics and state constraints by polynomial functions
and performs \gls{sos}-based synthesis on the resulting
polynomial model~\cite{DuanControl}; the
approximation error is not explicitly included in the
certificate conditions. A different approach partitions the
operating domain and fits local polynomial models on the
resulting subdomains, enabling \gls{sos}-based analysis of the
identified piecewise-polynomial model~\cite{Cunis_aircraft}.
For guarantees with respect to the original non-polynomial
dynamics, the approximation error can instead be enclosed within
certified bounds, yielding an uncertain polynomial
model~\cite{YANG2020100837,Wajid_successive}.
The barrier
conditions must then hold for all admissible approximation
errors and, for piecewise models, in every local region, increasing the computational burden.
For periodic states represented in Euclidean coordinates,
equivalent boundary points must also be identified, which can be
enforced as a deterministic reset condition analogous to hybrid
barrier-certificate formulations~\cite{prajna_hybrid}.

Existing \gls{sos}-based synthesis methods typically represent
the certified controlled-invariant region using a single
polynomial barrier, which may be iteratively enlarged or
refined~\cite{8431617,10176323,toulkani2025minimally,Wang2023,
DuanControl}. In contrast, the sample-guided method
of~\cite{Wajid_successive} successively synthesizes
multiple barriers for fixed admissible controls, using a convex
\gls{sos} subproblem at each iteration to cover additional
states and enlarge the certified region.

A separate line of work uses neural barrier certificates,
which provide flexible nonlinear certificate parameterizations
~\cite{Peruffo2021NeuralBarrier,Zhao2023NeuralBarrier,
Edwards2024Fossil}. However, formal safety guarantees
for neural certificates require verification over the
continuous state domain, whose scalability remains challenging. For example, FOSSIL~2.0~\cite{Edwards2024Fossil}, which jointly searches for
certificate and controller candidates, is reported in~\cite{Wajid_successive} not to terminate within the
\(10^{4}\,\mathrm{s}\) timeout on either the \gls{CT} or
\gls{PM} benchmark.

Building on the classical polynomial recasting approach
of~\cite{Papachristodoulou2005}, originally developed for
Lyapunov analysis, algebraic lifting of trigonometric terms has
more recently been used in \gls{sos}-based \gls{cbf}
verification and synthesis with input constraints
~\cite{10384199,10885943,10752383}.
In particular, \cite{10885943} prescribes a positive
exponential \gls{cbf} rate and deploys the synthesized
certificates in a \gls{clf}-\gls{cbf}-\gls{qp} controller.
However, for non-polynomial robotic dynamics, the effect of choosing an exact lifted representation rather than a certified piecewise-polynomial representation on barrier synthesis remains unclear.

An alternative is to synthesize the barrier directly in
discrete time, rather than adapting a continuous-time
certificate for sampled-data control. In continuous-time
\gls{bbs}, fixing the barrier makes the invariance condition
affine in the control-witness coefficients~\cite{Wang2023}.
In discrete time, however, the control witness determines the
successor state, so evaluating the barrier at that state
generally introduces nonlinear dependence on these
coefficients. For polynomial systems,
\cite{shakhesi2025synthesisdiscretetimecontrolbarrier}
addresses this using auxiliary variables and additional
\gls{sos} conditions, while
\cite{fochesato2025synthesissafetycertificatesdiscretetime}
obtains convex joint synthesis for uncertain linear systems
under quadratic-barrier and linear-feedback restrictions.
General discrete-time barrier synthesis for nonlinear systems
remains challenging; moreover, discretizing lifted dynamics
need not preserve the lifting-induced algebraic constraints.
\section{Preliminaries and Problem Formulation}\label{sec:problem}

\subsection*{Notation}

All vectors are column vectors. The ring of real polynomials in
$x$ is denoted by $\mathbb{R}[x]$, and the cone of
\gls{sos} polynomials in $x$ by $\Sigma[x]$. The same
convention is used for other indeterminates and variable tuples,
such as $\mathbb{R}[z]$. %and $\mathbb{R}[x,u]$.
Moreover, $\mathbb{R}[x]^r$ denotes the set of
$r$-dimensional polynomial vectors in $x$, while
$\Sigma[x]^r$ denotes the set of $r$-dimensional vectors
whose entries are \gls{sos} polynomials.

For a set $\mathcal{A}\subseteq\mathbb{R}^n$,
$\partial\mathcal{A}$ and
$\operatorname{int}(\mathcal{A})$ denote its boundary and
interior, respectively. For a differentiable scalar function
$p\colon\mathbb{R}^n\to\mathbb{R}$, $\nabla p(x)$ denotes
its gradient; when $p$ is twice differentiable,
$\nabla^2p(x)$ denotes its Hessian.

For a scalar function $p$ and a set $\mathcal{A}$, the
notation $p \ge_{\mathcal{A}} 0$
means that $p(x)\ge 0$ for all $x\in\mathcal{A}$.
Inequalities between vectors are interpreted entry-wise.

\subsection{System Model and Safety Constraints}
\label{sec:system_model}

We consider the continuous-time nonlinear control-affine system
\begin{equation}
\label{eq:dynamics}
    % \Psi\colon\quad
    % \dot{x}=f(x)+g(x)u,
    \dot{x}=f(x)+g(x)u,
\end{equation}
where $x\in\mathbb{R}^{n_x}$ is the state and
$u\in\mathbb{R}^{m}$ is the control input. The maps
$f\colon\mathbb{R}^{n_x}\to\mathbb{R}^{n_x}$ and
$g\colon\mathbb{R}^{n_x}\to\mathbb{R}^{n_x\times m}$
are continuously differentiable on an open set containing the
domain of interest and may contain non-polynomial terms.

% The domain of interest is a compact basic
% semi-algebraic set
% $\mathcal{D}\subset\mathbb{R}^{n_x}$.
A basic semi-algebraic set is described by finitely many
polynomial inequality and equality constraints. The domain of
interest is a compact basic semi-algebraic set
\(\mathcal{D}\subset\mathbb{R}^{n_x}\).
Within this domain,
the state is required to remain in the safe set
\begin{equation}
\label{eq:safe_set}
    \mathcal{S}
    :=
    \left\{
        x\in\mathcal{D}
        \;\middle|\;
        s(x)\geq 0,\;
        h(x)=0
    \right\},
\end{equation}
where $s\in\mathbb{R}[x]^{n_s}$ and
$h\in\mathbb{R}[x]^{n_h}$ are vectors of polynomial
inequality and equality constraints, respectively. The
algebraic relations $h(x)=0$ are assumed to be preserved by
the dynamics~\eqref{eq:dynamics}.

The control input is constrained to the nonempty compact
polytope
\begin{equation}
\label{eq:input_set}
    \mathcal{U}
    :=
    \left\{
        u\in\mathbb{R}^{m}
        \;\middle|\;
        F u\leq e
    \right\},
\end{equation}
where $F\in\mathbb{R}^{r\times m}$ and
$e\in\mathbb{R}^{r}$.

\subsection{Piecewise-Polynomial Approximation}
\label{sec:polynomial_representations}

To obtain a polynomial representation in the original state
coordinates, the domain $\mathcal{D}$ is partitioned into
finitely many basic semi-algebraic regions
$\mathcal{I}_q$, $q=1,\ldots,Q$. On each region, the
non-polynomial terms in $f$ and $g$ are replaced by
polynomial approximations together with bounded approximation
residuals. On each region $\mathcal{I}_q$, the resulting
polynomial differential inclusion is written as
\begin{equation}
\label{eq:pw_prelim_dynamics}
\begin{aligned}
    % \tilde{\Psi}_q\colon\quad
    \dot{x}
    % &=
    % \tilde F_q(x,u,\delta) \\
    &=
     f_q(x,\delta)
    +
    g_q(x,\delta)u,
    \qquad
    \delta\in{\Delta}_q ,
\end{aligned}
\end{equation}
for $x\in\mathcal{I}_q$, where
$ f_q\in\mathbb{R}[x,\delta]^{n_x}$ and
$ g_q\in\mathbb{R}[x,\delta]^{n_x\times m}$.
For each fixed $\delta$, the model is affine in $u$, and
for each fixed $u$, it is affine in $\delta$. The compact
set ${\Delta}_q$ is selected such that

\begin{equation}
\label{eq:pw_model_inclusion}
\begin{aligned}
    f(x)+g(x)u
    &\in
    \bigl\{
        f_q(x,\delta) + g_q(x,\delta)u
        \,\big|\,
        \delta\in {\Delta}_q
    \bigr\},
    \\
    &\hspace{-12mm}
    \forall x\in\mathcal{I}_q,\quad
    \forall u\in\mathcal{U}.
\end{aligned}
\end{equation}

Thus,~\eqref{eq:pw_prelim_dynamics} provides a polynomial
over-approximation of the original dynamics on
$\mathcal{I}_q$.

% \subsection{Control Barrier Certificates}
\subsection{Barrier Certificates for Controlled Invariance}
\label{sec:cbc}

We recall the barrier-certificate conditions for
controlled invariance. Let
$B\colon\mathbb{R}^{n_x}\to\mathbb{R}$ be continuously
differentiable, and define its nonnegative superlevel set over
the domain of interest as
\begin{equation}
\label{eq:candidate_set}
    \mathcal{C}
    :=
    \left\{
        x\in\mathcal{D}
        \;\middle|\;
        B(x)\geq0
    \right\}.
\end{equation}

\begin{definition}[Barrier Certificate]
\label{def:rcis}
The function $B$ together with a state-feedback control witness
$u_B\colon\mathcal{C}\cap\mathcal{S}\to\mathbb{R}^{m}$ is a
\emph{barrier certificate} for
$\mathcal{S}$ on $\mathcal D$ if
the following conditions hold.

\begin{enumerate}%[label=\textnormal{(C\arabic*)}]
    \item
    \label{cond:safety}
    \textbf{Safety:}
    The candidate states satisfying the algebraic state
    constraints are contained in the safe set:
    \begin{equation*}
    % \label{eq:C1}
        \left\{
            x\in\mathcal{C}
            \;\middle|\;
            h(x)=0
        \right\}
        \subseteq
        \mathcal{S}.
    \end{equation*}
    The restriction $h(x)=0$ is included because
    $\mathcal{C}$ is defined over the ambient domain
    $\mathcal{D}$, whereas only states satisfying the algebraic
    constraints are admissible system states.
    \item
    \label{cond:input}
    \textbf{Input feasibility:}
    There exists a state-feedback control witness
    $u_B\colon\mathcal{C}\cap\mathcal{S}\to\mathbb{R}^{m}$
    that satisfies the input constraints throughout the certified
    set:
    \begin{equation*}
    % \label{eq:C2}
        u_B(x)\in\mathcal{U},
        \qquad
        \forall x\in\mathcal{C}\cap\mathcal{S}.
    \end{equation*}
    \item
    \label{cond:invariance}
    \textbf{Barrier-boundary invariance (Nagumo~\cite{Nagumo1942berDL}):}
On \(B=0\) within \(\mathcal S\), the control witness satisfies
\begin{equation*}
    \nabla B(x)^\top
    \bigl(f(x)+g(x)u_B(x)\bigr)\geq0 .
\end{equation*}
\end{enumerate}
\end{definition}

Condition~\ref{cond:safety} ensures that the candidate set does not extend beyond the prescribed safety
constraints. Condition~\ref{cond:input} requires the control
witness to satisfy the control input bound throughout
\(\mathcal{C}\cap\mathcal{S}\).
Condition~\ref{cond:invariance} imposes the Nagumo boundary
condition on \(B=0\) within \(\mathcal S\). Under the usual
regularity assumptions, this preserves \(B(x(t))\geq0\) as
long as \(x(t)\in\mathcal D\)~\cite{Nagumo1942berDL,blanchini1999set}. 

\subsection{SOS Programming and Positivity Certificates}
\label{sec:sos_prelim}

Many constraints in this paper reduce to verifying
polynomial nonnegativity over semi-algebraic sets.
We employ \gls{sos} programming to encode such
conditions as tractable semidefinite programs.

A polynomial $p \in \mathbb{R}[x]$ is an
\gls{sos} polynomial, written
$p \in \Sigma[x]$, if there exist polynomials
$p_1,\dots,p_r \in \mathbb{R}[x]$ such that
$p(x) = \sum_{i=1}^{r} p_i(x)^2$. 
Any \gls{sos} polynomial is globally nonnegative, and
checking whether a polynomial of fixed degree is
\gls{sos} reduces to a semidefinite feasibility
problem~\cite{parrilo2000sos}.
% Any \gls{sos} polynomial
% is globally nonnegative, and tildeing whether a
% polynomial of fixed degree is \gls{sos} reduces to a
% semidefinite program~\cite{parrilo2000sos}.

To certify nonnegativity over a semi-algebraic set
rather than globally, we use the Positivstellensatz.

\begin{theorem}[Putinar's Positivstellensatz
{\cite{putinar1993positive}}]\label{thm:putinar}
Let $\mathcal{K} := \{x \in \mathbb{R}^{n_x} \mid
\kappa(x) \geq 0,\;\; \ell(x) = 0\}$ be a compact
semi-algebraic set defined by
$\kappa \in \mathbb{R}[x]^{n_\kappa}$ and
$\ell \in \mathbb{R}[x]^{n_\ell}$, whose associated
quadratic module is Archimedean. If $p(x) > 0$ for all
$x \in \mathcal{K}$, then there exist
$\sigma_0 \in \Sigma[x]$,
$\sigma \in \Sigma[x]^{n_\kappa}$, and
$\eta \in \mathbb{R}[x]^{n_\ell}$ such that
\begin{equation}\label{eq:putinar}
    p(x) = \sigma_0(x) + \sigma(x)^\top \kappa(x)
    + \eta(x)^\top \ell(x).
\end{equation}
\end{theorem}
Any decomposition of the form~\eqref{eq:putinar} directly
certifies $p\geq_{\mathcal K}0$. Under the assumptions of
Theorem~\ref{thm:putinar}, such a decomposition is guaranteed
to exist for every polynomial strictly positive on
$\mathcal K$, provided no degree bounds are imposed on the
multipliers.

In the synthesis problems considered here, the polynomial
$p$ and the multipliers in~\eqref{eq:putinar} may both
depend on unknown coefficients. The identity is enforced by
matching polynomial coefficients, while the \gls{sos}
constraints are represented through positive-semidefinite Gram
matrices. With fixed degree bounds, the resulting problem is
finite-dimensional and becomes a \gls{sdp} whenever all
decision variables enter affinely. 

\subsection{Problem Formulation}
\label{sec:problem_formulation}

Given the non-polynomial control-affine
system~\eqref{eq:dynamics}, the domain \(\mathcal D\), the safe
set \(\mathcal S\) in~\eqref{eq:safe_set}, the input set
\(\mathcal U\) in~\eqref{eq:input_set}, and a sampling period
\(T_s>0\), the objectives are to:

\begin{enumerate}
\item construct a barrier certificate defining a certified region
\(\mathcal C\subseteq\mathcal S\) that is controlled
invariant for the original non-polynomial dynamics, i.e.,
for every \(\mathbf{x}(t_0)\in\mathcal C\), there exists an
admissible control \(u(t)\in\mathcal U\) such that
\[
    \mathbf{x}(t)\in\mathcal C,\qquad \forall t\ge t_0;
\]

\item given a barrier certificate, construct a
sampled-data safety filter that, whenever feasible at \(t_k\),
prevents the state from leaving \(\mathcal C\) during the
subsequent \gls{zoh} interval:
\[
    \mathbf{x}(t_k)\in\mathcal C
    \;\Longrightarrow\;
    \mathbf{x}(t)\in\mathcal C,
    \qquad
    \forall t\in(t_k,t_{k+1}].
\]
\end{enumerate}
\section{Joint Barrier Synthesis}
\label{sec:bilinear_method}
This section develops \gls{ldbb} on the polynomial 
synthesis framework of~\cite{Wang2023} and uses the same
framework to construct \gls{pwbb} as the certified
piecewise-polynomial baseline.
Following the recasting procedure
of~\cite{Papachristodoulou2005}, the non-polynomial terms are
represented by auxiliary variables whose dynamics are obtained
using the chain rule. Together with the lifting-induced
polynomial equalities and inequalities, the resulting augmented
polynomial dynamics define the lifted model over the lifted-state domain
$\hat{\mathcal D}$. When initialized consistently with the
lifting map, the projection of its trajectories onto the
original state coordinates reproduces the original
continuous-time dynamics.

Using this lifted model, we jointly synthesize a lifted
polynomial barrier and a lifted polynomial control witness by
enforcing the \gls{sos} conditions over
$\hat{\mathcal D}$. Because the resulting bilinear
\gls{sos} problem is nonconvex, we solve it through alternating
optimization, fixing one block of decision variables at a time
so that each subproblem is convex. The synthesized certificate
is subsequently evaluated in the original coordinates through
the lifting map and used to construct the sampled-data safety
filter developed in ~\secref{sec:zoh_safety_filter}.

For comparison, we also formulate \gls{pwbb} as a 
baseline. It uses the same overall certificate structure and alternating
\gls{sos} synthesis framework, but replaces the lifted model with a
piecewise-polynomial representation of the original dynamics with certified
approximation-error bounds. This alignment allows us to compare how the
lifted and piecewise model representations affect the synthesized certificates
and their closed-loop behavior.

% \subsection{Approximated model}
\subsection{Polynomial Lifting}\label{sec:lifting}
 We adopt the lifting (recasting) approach 
of~\cite{Papachristodoulou2005}, originally developed 
for Lyapunov stability analysis, and extend it to the 
barrier certificate synthesis setting.
When the dynamics in~\eqref{eq:dynamics} contain
non-polynomial terms in the drift vector field $f(x)$ or
the input matrix $g(x)$, we introduce auxiliary variables
($\tilde{z}\in\mathbb{R}^{n_{\tilde{z}}}$) to represent these terms. The lifted state is defined as
\begin{equation}\label{eq:augmented_state}
    z := \phi(x)
    =
    \begin{bmatrix}
        x\\
        \psi(x)
    \end{bmatrix}
    =
    \begin{bmatrix}
        x\\
        \tilde{z}
    \end{bmatrix}
    \in \mathbb{R}^{n_z},
    \quad
    \tilde{z}:=\psi(x).
\end{equation}
where $\psi:\mathbb{R}^{n_x}\to\mathbb{R}^{n_{\tilde{z}}}$ is a generally non-polynomial lifting map and $n_z = n_x+n_{\tilde{z}}$. 
The auxiliary dynamics follow from
the chain rule:
\begin{equation}
\label{eq:aux_derivative}
    \dot{\tilde z}
    =
    \frac{\partial \psi}{\partial x}(x)\dot{x}.
\end{equation}
We consider systems for which substituting the original
dynamics into~\eqref{eq:aux_derivative} and repeatedly
replacing the selected non-polynomial expressions by their
corresponding auxiliary variables defined by $\psi$ expresses
all augmented-state derivatives as polynomials in $z$, while
preserving affine dependence on $u$.
This yields the lifted polynomial control-affine model
\begin{equation}\label{eq:lifted_dynamics}
    % \hat{\Psi}\colon\quad
    \dot{z} = \hat{f}(z) + \hat{g}(z)\,u,
\end{equation}
where $\hat f:\mathbb{R}^{n_z}\to\mathbb{R}^{n_z}$ and
$\hat g:\mathbb{R}^{n_z}\to
\mathbb{R}^{n_z\times m}$ are polynomial maps. We focus on this case throughout; rational lifted
representations are treated in~\cite{Papachristodoulou2005}.

\subsubsection{Lifting-Induced Algebraic Set}
\label{sec:lifting_induced}

The auxiliary variables satisfy the defining relation
$\tilde z=\psi(x)$. Since this relation is generally
non-polynomial, it cannot be imposed directly in an \gls{sos}
program. However, it implies polynomial equality and inequality
constraints of the form
\begin{equation}
\label{eq:lifting_relations}
    \hat h(z)=0,
    \qquad
    \hat s(z)\geq 0,
\end{equation}

where
$\hat h\in\mathbb{R}[z]^{n_{\hat h}}$ and
$\hat s\in\mathbb{R}[z]^{n_{\hat s}}$.
For example, the substitutions
$c=\cos\theta$ and $s=\sin\theta$ imply
$c^2+s^2-1=0$. Similarly, $y=e^x$ implies
$y\geq0$, whereas $y=\sqrt{x}$ implies
$y^2-x=0$ together with $y\geq0$.

These constraints define the lifting-induced algebraic set
\begin{equation}
\label{eq:lifting_algebraic_set}
    \hat{\mathcal K}
    :=
    \left\{
        z\in\hat{\mathcal D}
        \;\middle|\;
        \hat s(z)\geq0,\;
        \hat h(z)=0
    \right\},
\end{equation}
where $\hat{\mathcal D}\subset\mathbb{R}^{n_z}$ is a compact
semialgebraic lifted-state domain chosen such that
$\phi(\mathcal D)\subseteq\hat{\mathcal D}$.

Every state of the form $z=\phi(x)$ satisfies the constraints
in~\eqref{eq:lifting_relations} and therefore belongs to
$\hat{\mathcal K}$ in~\eqref{eq:lifting_algebraic_set}. The converse need not hold,
because the available polynomial relations may not completely
characterize the defining relation $\tilde z=\psi(x)$.
Consequently, enforcing the \gls{sos} conditions over
$\hat{\mathcal K}$ is valid for the original system,
but may introduce conservatism if the set contains additional
lifting-inconsistent states.

\begin{remark}[Class of liftable systems]
\label{rem:liftable} 
The recasting procedure applies to a broad class of systems
whose selected non-polynomial expressions admit a finite
rational closure under differentiation, so that their
derivatives can be represented using finitely many original
and auxiliary variables~\cite{Papachristodoulou2005}. This
includes systems containing trigonometric, exponential,
radical, and hyperbolic terms, as well as their nested
compositions. The present formulation considers the subclass
for which the resulting lifted dynamics are polynomial.
% Rational lifted dynamics
% are discussed in Remark~\ref{rem:rational}.
\end{remark}
% \ques{Do you think it's symbolic heavy? I should formulate using less symbols? Because of many method it was hard to separate will simple symbols.}
The following trajectory-consistency property follows directly
from the recasting construction of~\cite{Papachristodoulou2005}.

% \begin{proposition}[Trajectory consistency under the lifting map]
% \label{prop:trajectory_consistency}
% Consider the original system~\eqref{eq:dynamics} and its lifted
% representation~\eqref{eq:lifted_dynamics} under the same input
% $u(\cdot)$. If the lifted system is initialized consistently as
% $$
%     z(0)=\phi(x(0)),
% $$
% then, while the corresponding solutions exist,
% $$
%     z(t)=\phi(x(t)).
% $$
% Consequently, the original-state components of $z(t)$ satisfy
% \eqref{eq:dynamics}, and
% $$
%     z(t)\in\hat{\mathcal K}_{\hat\Psi}.
% $$
% \end{proposition}

% \begin{proof}
% Let $x(t)$ satisfy~\eqref{eq:dynamics} and define
% $z(t):=\phi(x(t))$. The original-state components satisfy
% \eqref{eq:dynamics}, while the auxiliary components satisfy
% \eqref{eq:aux_derivative} by the chain rule. Hence,
% $\phi(x(t))$ satisfies the lifted dynamics with initial
% condition $\phi(x(0))$. Uniqueness of solutions then gives
% $z(t)=\phi(x(t))$.
% \end{proof}

\begin{proposition}[Trajectory consistency under the lifting map]
\label{prop:trajectory_consistency}
Consider the original system~\eqref{eq:dynamics} and the
lifted system~\eqref{eq:lifted_dynamics} under the same input
\(u(\cdot)\). Assume that, for this input, the corresponding
initial-value problems admit unique solutions. Let \(x(t)\)
solve the original system with \(x(0)=x_0\), and let \(z(t)\)
solve the lifted system with the consistent initial condition
\[
    z(0)=\phi(x_0).
\]
Then, on their common interval of existence,
\[
    z(t)=\phi(x(t)).
\]
Consequently, the original-state components of \(z(t)\)
satisfy~\eqref{eq:dynamics}. Moreover, whenever
\(x(t)\in\mathcal D\),
\[
    % z(t)\in\hat{\mathcal K}_{\hat\Psi}.
    z(t)\in\hat{\mathcal K}.
\]
\end{proposition}

\begin{proof}
By the chain rule and the construction of the lifted dynamics,
the curve \(\phi(x(t))\) satisfies
\eqref{eq:lifted_dynamics} under the same input \(u(\cdot)\).
It also has the same initial condition as the lifted solution:
\[
    \phi(x(0))=\phi(x_0)=z(0).
\]
Uniqueness of the lifted initial-value problem therefore gives
\[
    z(t)=\phi(x(t))
\]
on their common interval of existence. Since the lifting map
retains \(x\) as the original-state components of \(z\), these
components coincide with \(x(t)\) and satisfy
\eqref{eq:dynamics}. Finally, by the definition of
\(\hat{\mathcal K}\), every lifted state
\(\phi(x)\) with \(x\in\mathcal D\) satisfies the
lifting-induced constraints. Hence,
\(z(t)\in\hat{\mathcal K}\) whenever
\(x(t)\in\mathcal D\).
\end{proof}

\paragraph{Example: Unicycle Model}

Consider the unicycle kinematics
\begin{equation*}
% \label{eq:unicycle}
    \begin{bmatrix}
        \dot{x}_1\\
        \dot{x}_2\\
        \dot{\theta}
    \end{bmatrix}
    =
    \begin{bmatrix}
        \cos\theta & 0\\
        \sin\theta & 0\\
        0 & 1
    \end{bmatrix}
    \begin{bmatrix}
        v\\
        \omega
    \end{bmatrix},
\end{equation*}
With $c:=\cos\theta$ and $s:=\sin\theta$, the non-polynomial terms are
represented by $\tilde z:=\psi(x)=[c,\;s]^\top$, giving the lifted state
$z=[x_1,\;x_2,\;\theta,\;c,\;s]^\top$.
% with state
% $x=[x_1,\;x_2,\;\theta]^\top$ and input
% $u=[v,\;\omega]^\top$.
% The non-polynomial terms are represented by
% \begin{equation*}
% % \label{eq:unicycle_aux}
%     \tilde z
%     :=
%     \psi(x)
%     =
%     \begin{bmatrix}
%         \cos\theta\\
%         \sin\theta
%     \end{bmatrix},
%     \qquad
%     z
%     =
%     \begin{bmatrix}
%         x_1 & x_2 & \theta & c & s
%     \end{bmatrix}^{\top},
% \end{equation*}
% where $c:=\cos\theta$ and $s:=\sin\theta$.
Applying the chain rule gives
\begin{equation}
\label{eq:unicycle_chain}
    \dot c=-s\omega,
    \qquad
    \dot s=c\omega.
\end{equation}
The lifted dynamics are therefore
% \begin{equation}
% \label{eq:unicycle_lifted}
%     \dot z
%     =
%     \begin{bmatrix}
%         cv\\
%         sv\\
%         \omega\\
%         -s\omega\\
%         c\omega
%     \end{bmatrix},
% \end{equation}
\begin{equation}
\label{eq:unicycle_lifted}
    \dot z
    =
    \begin{bmatrix}
        cv & sv & \omega & -s\omega & c\omega
    \end{bmatrix}^{\top},
\end{equation}
which are polynomial in the lifted state and input. The lifting also
implies the polynomial equality
\begin{equation}
\label{eq:unicycle_constraint}
    \hat h(z)
    :=
    c^2+s^2-1
    =
    0.
\end{equation}
Hence, for this example,
$\hat{\mathcal K}=\{z\in\hat{\mathcal D}\mid c^2+s^2=1\}$.
% Hence, for this example,
% $$
%     \hat{\mathcal K}
%     =
%     \left\{
%         z\in\hat{\mathcal D}
%         \;\middle|\;
%         c^2+s^2=1
%     \right\}.
% $$

\begin{remark}[Reduced representation of orientation]
\label{rem:reduced_orientation}
The angle $\theta$ is redundant when the dynamics, state constraints, and
certificate depend on orientation only through
$c=\cos\theta$ and $s=\sin\theta$. In this case, it can be removed and
the reduced lifted state can be written as
$$
    z_r
    :=
    \begin{bmatrix}
        x_1 & x_2 & c & s
    \end{bmatrix}^{\top},
$$
with dynamics obtained from~\eqref{eq:unicycle_lifted} by omitting the
$\dot\theta$ equation. 
Thus, orientation is represented directly on the unit circle. If an angular
constraint is needed, it can be imposed directly as a semialgebraic constraint
in $(c,s)$ whenever it describes a heading direction rather than an unwrapped
angle. The same idea extends to other configurations, e.g.,
$\mathrm{SO}(3)$ or $\mathrm{SE}(3)$, only when the constraints can be written
directly in terms of the chosen algebraic orientation variables.
\end{remark}

\subsubsection{Embedding of Original Constraints}

In the generic lifted state~\eqref{eq:augmented_state}, the original
state $x$ remains a component of $z$. The polynomial state
constraints in~\eqref{eq:safe_set} can therefore be represented in
the lifted coordinates.
% as
% \begin{equation}
% \label{eq:lifted_state_constraints}
%     \hat{s}(z):=s(x),
%     \qquad
%     \hat{h}(z):=h(x).
% \end{equation}
If a reduced lifted representation removes a redundant original
coordinate, as in Remark~\ref{rem:reduced_orientation}, any constraint
depending on that coordinate must instead be expressed in terms of the
retained lifted variables. We use such a reduced representation only when
the corresponding constraints admit an equivalent semialgebraic
description; otherwise, the original coordinate is retained.

The lifted domain used for the input-feasibility and invariance
certificates combines the original state constraints with the
lifting-induced algebraic constraints:
% \begin{multline}    
% \label{eq:full_lifted_set}
%     \hat{\mathcal{K}}^{+}
%     :=
%     \left\{
%          z\in\hat{\mathcal{D}}
%         \;\middle|\;
%         {s}(x)\geq0,\;
%         {h}(x)=0, \right. \\
%         \left.
%         \hat{s}(z)\geq0,\;
%         \hat{h}(z)=0
%     \right\}.
% \end{multline}
\begin{equation}
\label{eq:full_lifted_set}
    \hat{\mathcal{K}}^{+}
    :=
    \left\{
         z\in\hat{\mathcal{K}}
        \;\middle|\;
        {s}(x)\geq0,\;
        {h}(x)=0
    \right\}.
\end{equation}    

The admissible input set $\mathcal U$ remains unchanged in the lifted
formulation. Although the synthesized control witness $\hat u(z)$ is
parameterized by the lifted state, its value is a physical control input in
$\mathbb{R}^{m}$. 
% Therefore, the synthesis enforces
% $$
%     \hat u(z)\in\mathcal U,
%     \qquad z\in\hat{\mathcal K}.
% $$

% ==========================================================
\subsection{Barrier Synthesis with Lifted Dynamics}
\label{sec:lifted_cbc}
% ==========================================================

We now formulate the synthesis problem using the lifted
polynomial dynamics and the associated algebraic constraints. 
The control witness in Definition~\ref{def:rcis} is a general
state-feedback map without a prescribed finite-dimensional
parameterization.

For computational synthesis, we restrict the control witness to a
polynomial map
$\hat u\in\mathbb{R}[z]^m$ and synthesize it jointly with a polynomial
barrier $\hat B\in\mathbb{R}[z]$. The candidate certified set in the
lifted coordinates is
\begin{equation}
\label{eq:lifted_candidate_set}
    \hat{\mathcal C}
    :=
    \left\{
        % z\in\mathbb{R}^{{n_z}}
        z\in\hat{\mathcal D}
        \;\middle|\;
        \hat B(z)\geq0
    \right\}.
\end{equation}

The lifted barrier and polynomial control witness must satisfy the
following conditions.

\paragraph{Safety}
The certified states that satisfy the lifting-induced constraints and
the original equality constraints must also satisfy every original
state-safety inequality:
\begin{equation}
\label{eq:lifted_C1}
    \left\{
        z\in\hat{\mathcal K}
        \;\middle|\;
        h(x)=0,\;
        \hat B(z)\geq0
    \right\}
    \subseteq
    \left\{
        z
        \;\middle|\;
         s(x)\geq0
    \right\}.
\end{equation}

\paragraph{Input feasibility}
The control witness must satisfy the physical input constraints
throughout the candidate certified set:
\begin{equation}
\label{eq:lifted_C2}
    \hat u(z)\in\mathcal U,
    \qquad
    \forall z\in
    \hat{\mathcal C}\cap\hat{\mathcal K}^{+}.
\end{equation}

\paragraph{Barrier-boundary invariance}
On \(\hat B(z)=0\) within the lifted certification domain, the
control witness must satisfy
\begin{equation}
\label{eq:lifted_C3}
    \nabla\hat B(z)^\top
    \bigl(\hat f(z)+\hat g(z)\hat u(z)\bigr)
    \geq 0 .
\end{equation}

We next express
\eqref{eq:lifted_C1}--\eqref{eq:lifted_C3}
as \gls{sos} constraints. 

A direct \gls{sos} encoding of the safety
containment~\eqref{eq:lifted_C1} would use
$\hat B(z)\geq0$ as one of the polynomial inequalities defining the
proof domain. Since both $\hat B$ and its corresponding
\gls{sos} multiplier would be decision variables, this encoding would
introduce their product and hence an additional bilinearity.

We avoid this coupling by enforcing safety through a 
contrapositive condition. For each state-safety inequality
$s_i(x)\geq0$, define
\begin{equation*}
% \label{eq:lifted_unsafe_domain}
    \hat{\mathcal K}^{-}_{i}
    :=
    \left\{
        z\in\hat{\mathcal K}
        \;\middle|\;
         h(x)=0,\;
        - s_i(x)\geq0
    \right\}.
\end{equation*}
The safety condition is then imposed as
\begin{equation}
\label{eq:sos_lifted_safety}
    -\hat B(z)-\varepsilon_B
    \geq_{\hat{\mathcal K}^{-}_{i}}0,
    \qquad
    i=1,\ldots,n_s,
    \qquad
    \varepsilon_B>0.
\end{equation}
Thus, every state on the non-safe side of the $i$-th state constraint,
including its boundary, satisfies
$\hat B(z)\leq-\varepsilon_B$. Consequently, no such state can belong
to $\hat{\mathcal C}$, which is sufficient for the containment in
\eqref{eq:lifted_C1}. Unlike the direct containment encoding,
\eqref{eq:sos_lifted_safety} is jointly affine in the coefficients of
$\hat B$ and its multipliers.

For input feasibility, let $F_j$ denote the $j$-th row of $F$ and
let $e_j$ denote the corresponding component of $e$. For each
input-polytope facet $j=1,\ldots,r$, we impose
\begin{equation}
\label{eq:sos_lifted_input}
    e_j
    -F_j\hat u(z)
    -\hat\mu_j(z)\hat B(z)
    \geq_{{\hat{\mathcal K}}^{+}} 0,
    \qquad
    \hat\mu_j\in\Sigma[z].
\end{equation}
For any
$z\in\hat{\mathcal C}\cap\hat{\mathcal K}^{+}$,
both $\hat\mu_j(z)$ and $\hat B(z)$ are nonnegative, so
\eqref{eq:sos_lifted_input} implies
$
    F_j\hat u(z)\leq e_j,
$
and hence $\hat u(z)\in\mathcal U$ throughout the candidate certified
set.

Finally, the boundary-invariance condition is encoded as
\begin{equation}
\label{eq:sos_lifted_invariance}
    \nabla\hat B(z)^\top
    \bigl(
        \hat f(z)+\hat g(z)\hat u(z)
    \bigr)
    +\hat\lambda(z)\hat B(z)
    \geq_{{\hat{\mathcal K}}^{+}} 0.
\end{equation}
The polynomial $\hat\lambda\in\mathbb{R}[z]$ is unrestricted in sign. On $\hat B(z)=0$ within $\hat{\mathcal K}^{+}$, the multiplier term
vanishes and condition~\eqref{eq:sos_lifted_invariance} recovers the barrier-boundary requirement~\eqref{eq:lifted_C3}.

% ==========================================================
\subsection{Alternating Synthesis Algorithm}
\label{sec:algorithm_v2}
% ==========================================================

The lifted \gls{sos} program
\eqref{eq:sos_lifted_safety}--%
\eqref{eq:sos_lifted_invariance}
is not jointly convex. The input-feasibility and invariance
conditions contain the products
\(\hat\mu_j\hat B\), \(\hat\lambda\hat B\), and
\(\nabla\hat B^\top\hat g\hat u\), whereas the
contrapositive safety condition is affine in its decision
variables. 
We address this bilinearity through alternating
optimization. In the barrier step, with
\(\hat u_k\), \(\hat\lambda_k\), and
\(\hat\mu_{j,k}\) fixed, we solve over
\(\hat B_{k+1}\); in the control-witness step, with
\(\hat B_{k+1}\) fixed, we solve over
\(\hat u_{k+1}\), \(\hat\lambda_{k+1}\), and
\(\hat\mu_{j,k+1}\). Each resulting \gls{sos} subproblem is convex and is solved
as a feasibility problem.

To preserve the set certified at the previous iteration, the
barrier step additionally imposes
\begin{equation}
\label{eq:growth}
    \hat B_{k+1}(z)
    -
    \sigma_g(z)\hat B_k(z)
    \geq_{{\hat{\mathcal K}}^{+}} 0,
    \qquad
    \sigma_g\in\Sigma[z].
\end{equation}
Since \(\sigma_g\) is \gls{sos},
\eqref{eq:growth} implies that every state certified at
iteration \(k\) remains certified at iteration \(k+1\);
hence, the certified sets are nested across iterations.

An iteration is accepted only if both subproblems are feasible;
otherwise, the certificate from the last completed iteration is
returned. The complete procedure is summarized in
Algorithm~\ref{alg:alternation_v2}.

\begin{algorithm}[t]
\caption{\gls{ldbb} synthesis by alternating optimization}
\label{alg:alternation_v2}
\begin{algorithmic}[1]
\REQUIRE Lifted dynamics~\eqref{eq:lifted_dynamics},
domain ${\hat{\mathcal K}}^{+}$ in~\eqref{eq:full_lifted_set}, input set $\mathcal U$,
polynomial degrees, admissible initial witness $\hat u_0$,
initial multiplier $\hat\lambda_0$, and iteration limit
\ENSURE Lifted certificate
$(\hat B,\hat u,\hat\lambda)$

\STATE Set $k\leftarrow0$ and
$\hat\mu_{j,0}\leftarrow0$, $j=1,\ldots,r$

\STATE With $\hat u_0$, $\hat\lambda_0$, and
$\hat\mu_{j,0}$ fixed, solve
\eqref{eq:sos_lifted_safety}--\eqref{eq:sos_lifted_invariance}
for $\hat B_0$

\IF{the initialization problem is infeasible}
    \STATE \textbf{return failure}
\ENDIF

\WHILE{$k$ is below the iteration limit}

    \STATE With
    $\hat u_k$, $\hat\lambda_k$, and
    $\hat\mu_{j,k}$ fixed, solve
    \eqref{eq:sos_lifted_safety},
    \eqref{eq:sos_lifted_input},
    \eqref{eq:sos_lifted_invariance}, and
    \eqref{eq:growth}
    for $\hat B_{k+1}$ and $\sigma_g$

    \IF{the barrier step is infeasible}
        \STATE \textbf{return}
        $(\hat B_k,\hat u_k,\hat\lambda_k)$
    \ENDIF

    \STATE With $\hat B_{k+1}$ fixed, solve
    \eqref{eq:sos_lifted_input} and
    \eqref{eq:sos_lifted_invariance}
    for $\hat u_{k+1}$, $\hat\lambda_{k+1}$, and
    $\hat\mu_{j,k+1}$, $j=1,\ldots,r$

    \IF{the control-witness step is infeasible}
        \STATE \textbf{return}
        $(\hat B_k,\hat u_k,\hat\lambda_k)$
    \ENDIF

    \STATE $k\leftarrow k+1$

\ENDWHILE

\RETURN $(\hat B_k,\hat u_k,\hat\lambda_k)$
\end{algorithmic}
\end{algorithm}

% ==========================================================
\subsubsection{Recovery in the Original Coordinates}
\label{sec:recovery}
% ==========================================================

Let
$(\hat B,\hat u,\hat\lambda)$ be a feasible solution of the lifted
synthesis problem. The corresponding functions in the original
coordinates are obtained by composition with the lifting map:
\begin{equation}
\label{eq:pullback_function}
    B(x):=\hat B(\phi(x)),
    \
    u_B(x):=\hat u(\phi(x)),
    \
    \lambda(x):=\hat\lambda(\phi(x)).
\end{equation}
By Proposition~\ref{prop:trajectory_consistency}, these functions define
the corresponding barrier, control witness, and multiplier along the
trajectories of the original system. 

\begin{remark}[Advantage of continuous-time synthesis]
\label{rem:discrete_drift}
A direct discrete-time synthesis based on a numerical
discretization of the lifted dynamics raises two difficulties:
the barrier-control composition remains nonconvex even when
the barrier is fixed, and the discretized dynamics may not
preserve the lifting-induced algebraic constraints.

To illustrate the first issue, consider the forward-Euler
discrete-time invariance condition
\[
    \hat B\!\left(
        z_k+T_s\bigl(
        \hat f(z_k)+\hat g(z_k)\hat u(z_k)
        \bigr)
    \right)\geq0,
    \qquad
    z_k\in\hat{\mathcal C}\cap\hat{\mathcal K}^{+}.
\]
This condition composes the barrier with the
control-dependent successor state, causing the unknown
coefficients of \(\hat u\) to enter nonlinearly through the
monomials of \(\hat B\). Consequently, the control-witness
subproblem remains nonconvex even when \(\hat B\) is fixed,
and the standard continuous-time alternating procedure does
not directly resolve this composition issue~\cite{shakhesi2025synthesisdiscretetimecontrolbarrier}.
By contrast, the continuous-time invariance condition
\eqref{eq:sos_lifted_invariance} is affine in
\(\hat u\) and \(\hat\lambda\) when \(\hat B\) is fixed.

The second issue concerns preservation of the
lifting-induced algebraic constraints. 
For instance, applying forward Euler to the auxiliary dynamics
in~\eqref{eq:unicycle_chain}
gives
\[
    c_{k+1}^2+s_{k+1}^2
    =
    1+\omega_k^2T_s^2
\]
% when \(c_k^2+s_k^2=1\).
when the lifting constraint~\eqref{eq:unicycle_constraint}
holds at \(k\).
Thus, a state on the lifting manifold
generally produces a successor outside it. We therefore
synthesize the certificate using the exact continuous-time
lifted dynamics and account for sampled-data \gls{zoh}
implementation separately in
Section~\ref{sec:zoh_safety_filter}.
\end{remark}

\subsection{Barrier Synthesis with Piecewise-Polynomial Dynamics}
\label{sec:pwbb_method}

\gls{pwbb} applies similar \gls{bbs} conditions in the
original coordinates using the piecewise-polynomial approximation model in~\eqref{eq:pw_prelim_dynamics}. 
%% for sector ie lambda(x)
A barrier \(B\) and control witness \(u_B\) are shared
across all sectors, while the invariance multiplier
\(\lambda_q\in\mathbb{R}[x]\), \(q=1,\ldots,Q\), is
sector dependent.
For each sector, define
$$
    \mathcal K_q:=\mathcal S\cap\mathcal I_q,
    \qquad q=1,\ldots,Q.
$$

For each state-safety constraint, define
\begin{equation}
\label{eq:pwbb_unsafe_domain}
    \mathcal K^{-}_{q,i}
    :=
    \left\{
        x\in\mathcal D\cap\mathcal I_q
        \;\middle|\;
        h(x)=0,\;
        -s_i(x)\geq0
    \right\}.
\end{equation}

%% for sector wise lambda
The counterparts of
\eqref{eq:sos_lifted_safety}--%
\eqref{eq:sos_lifted_invariance} are
\begin{subequations}
\label{eq:pwbb_sos_conditions}
\begin{align}
    - B(x)-\varepsilon_B
    &\geq_{\mathcal K^{-}_{q,i}}0,
    \label{eq:pwbb_C1}
    \\
    e_j-F_j u_B(x)
    -\mu_j(x) B(x)
    &\geq_{\mathcal K_q}0,
    \label{eq:pwbb_C2}
\end{align}
\begin{equation}
\begin{aligned}
    &\nabla B(x)^\top
     \bigl( f_q(x,\delta) + g_q(x,\delta)u_B(x) \bigr)
    \\
    &\qquad
    +\lambda_q(x) B(x)
    \geq_{\mathcal K_q\times\Delta_q}0 .
\end{aligned}
\label{eq:pwbb_C3}
\end{equation}
\end{subequations}
Here,
$i=1,\ldots,n_s$, $j=1,\ldots,r$, and
$q=1,\ldots,Q$. The multipliers
$\mu_j\in\Sigma[x]$ are shared across sectors, whereas
$\lambda_q\in\mathbb R[x]$ is synthesized separately
on each sector.

\paragraph{Periodic-boundary consistency}
Because the same global polynomial \(B\) is used in all sectors,
no additional continuity constraints are required at internal
sector boundaries. However, for a periodic coordinate
\(\theta\in[-\pi,\pi]\), the points
\((\bar x,\pi)\) and \((\bar x,-\pi)\) represent the same
physical state. This identification can be viewed as a
deterministic coordinate reset, analogous to reset transitions
in hybrid barrier-certificate formulations~\cite{prajna_hybrid}.
Since the two endpoints represent the same physical state, we
use the value-preserving condition
\begin{equation}
\label{eq:pwbb_wrap_C0}
     B(\bar x,\pi)
    =
     B(\bar x,-\pi),
    \qquad \forall\,\bar x .
\end{equation}
This preserves the barrier sign, and hence certified-set
membership, across the coordinate reset.
For
$$
     B(\bar x,\theta)
    =
    \sum_{\alpha,k}
    b_{\alpha k}\bar x^\alpha\theta^k,
$$
this is enforced by the linear constraints
\begin{equation}
\label{eq:pwbb_wrap_coeff}
    \sum_{\substack{k\ \mathrm{odd}}}
    2\pi^k b_{\alpha k}
    =
    0,
    \qquad \forall\,\alpha .
\end{equation}

\begin{figure}[t]
    \centering
    \begin{subfigure}[b]{0.48\linewidth}
        \centering
        \includegraphics[width=\linewidth]{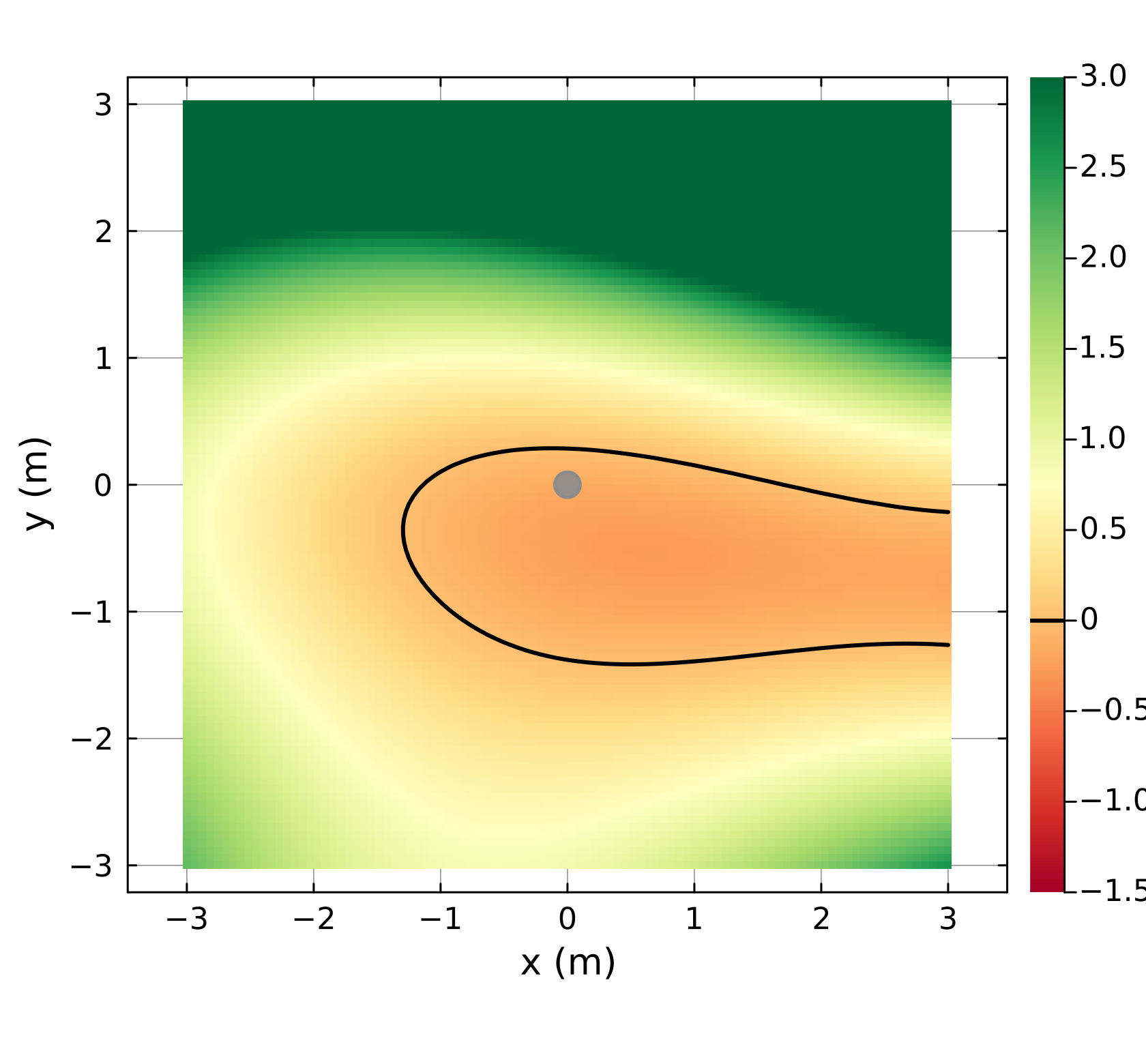}
        \subcaption{Without consistency, $\theta=\pi$}
    \end{subfigure}\hfill
    \begin{subfigure}[b]{0.48\linewidth}
        \centering
        \includegraphics[width=\linewidth]    {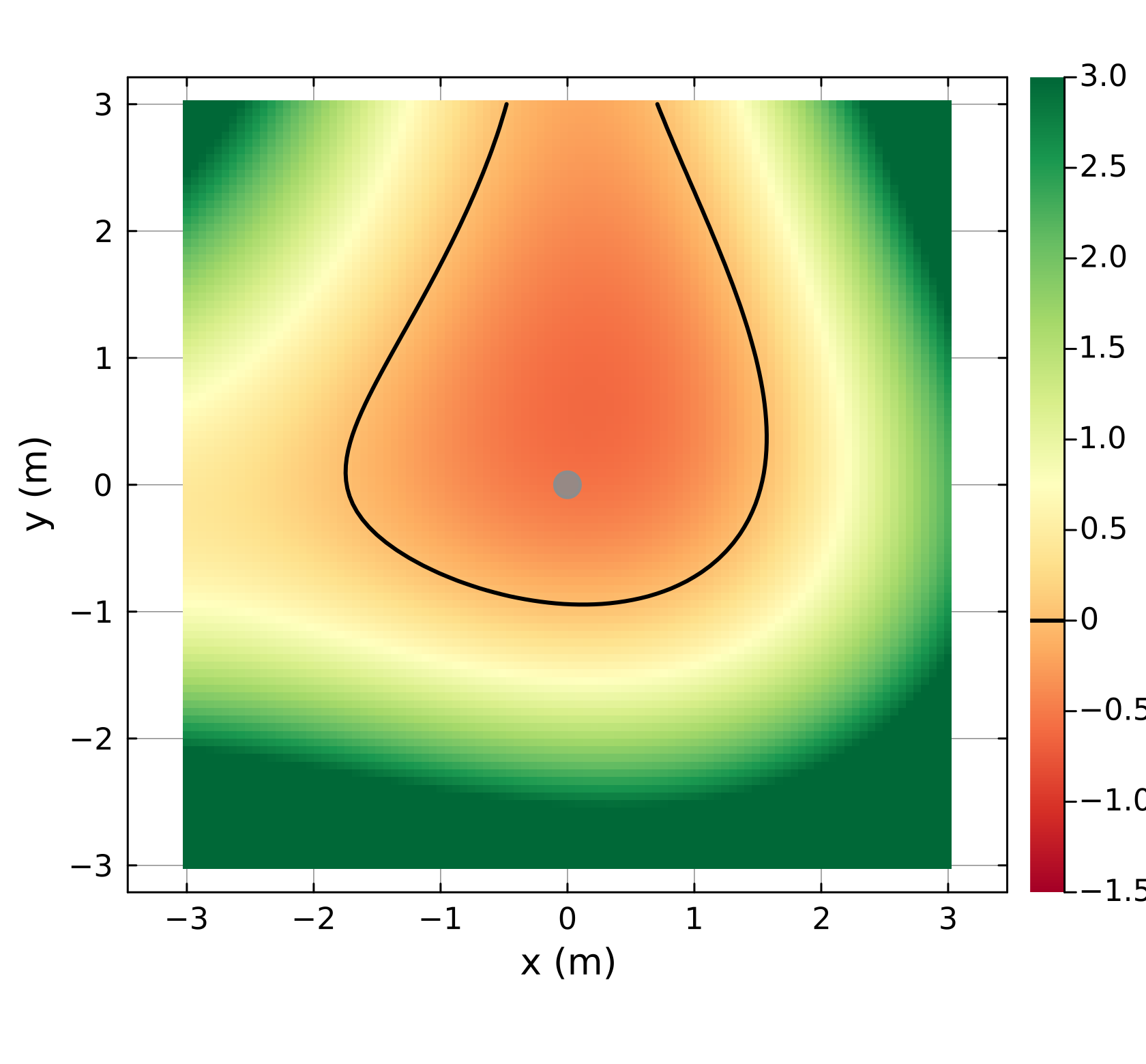}
        \subcaption{Without consistency, $\theta=-\pi$}
    \end{subfigure}
    \vspace{0.4em}
    \begin{subfigure}[b]{0.48\linewidth}
        \centering
        \includegraphics[width=\linewidth]
        {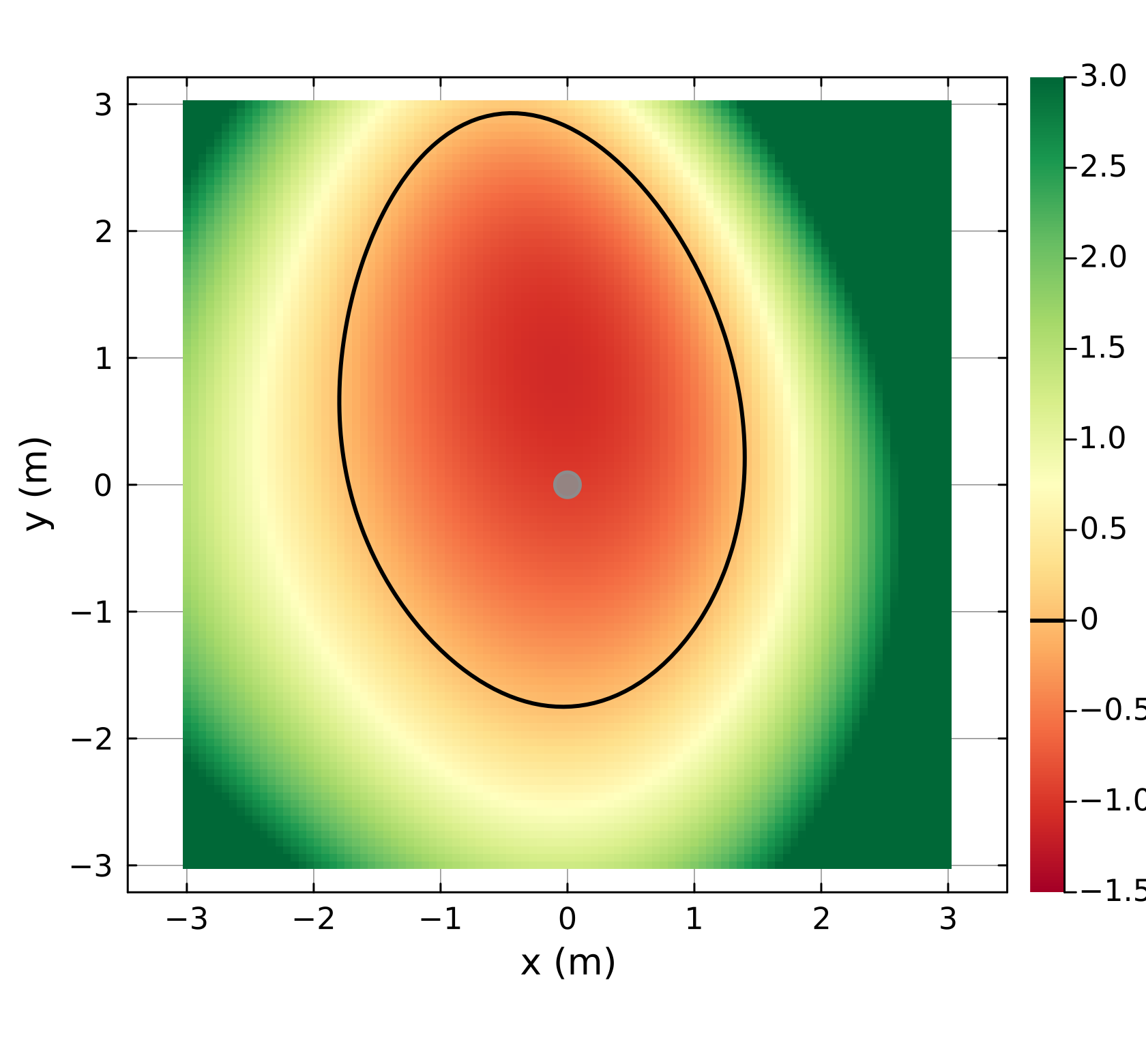}
        \subcaption{With consistency, $\theta=\pi$}
    \end{subfigure}\hfill
    \begin{subfigure}[b]{0.48\linewidth}
        \centering
        \includegraphics[width=\linewidth]{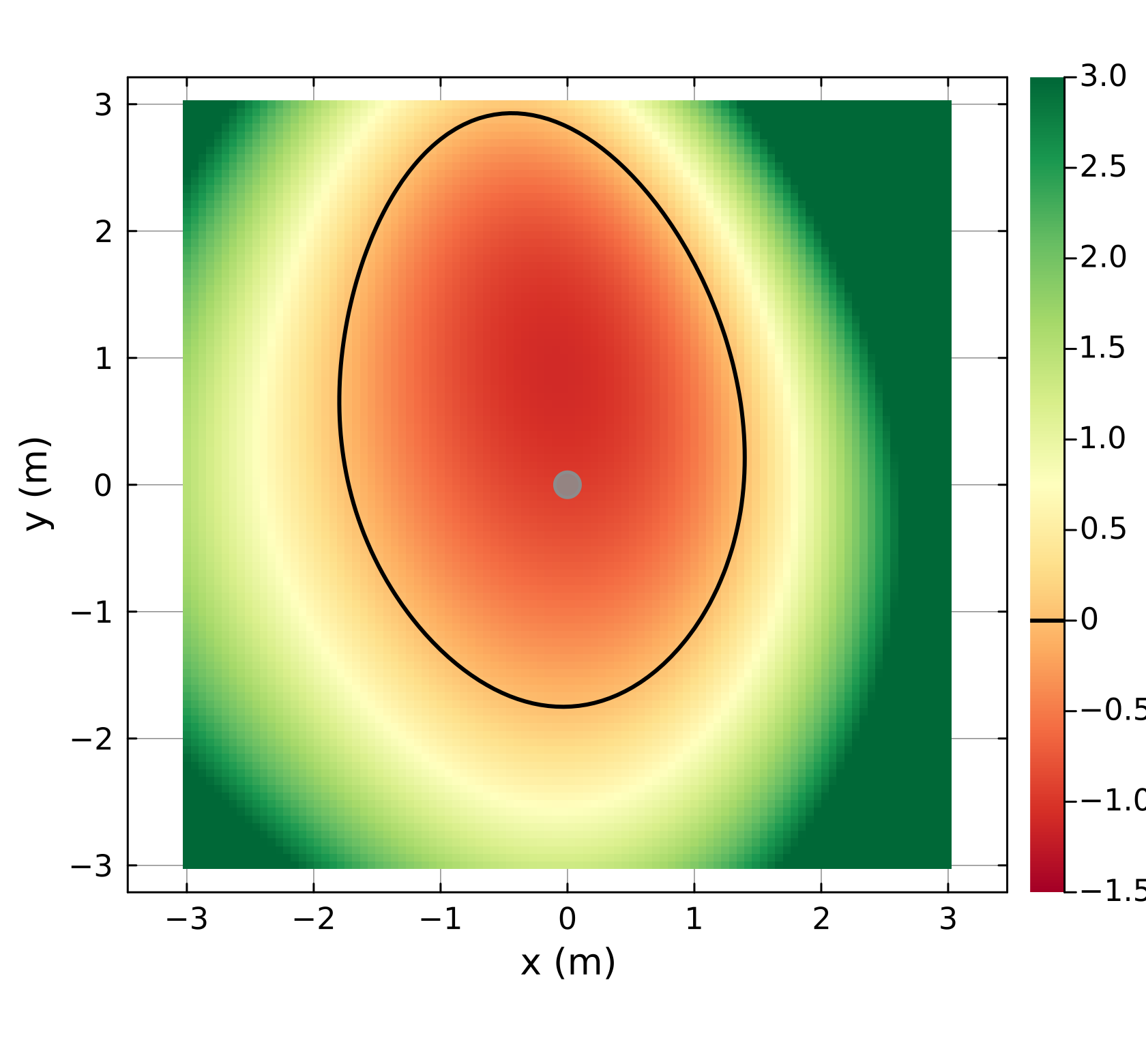}
        \subcaption{With consistency, $\theta=-\pi$}
    \end{subfigure}

    \caption{\acrshort{pwbb} slices for \acrshort{CT} model at $\theta=\pi$ (left) and
    $\theta=-\pi$ (right) and rest at: $v=0.3,\ \omega=0.0$, without periodic consistency (top)
    and with \eqref{eq:pwbb_wrap_C0} (bottom). Gray denotes the
    obstacle, the black contour denotes $ B(x)=0$, and
    $ B(x)\geq0$ is the certified set.}
    \label{fig:pwbb_wrap_consistency}
\end{figure}

Condition~\eqref{eq:pwbb_wrap_C0} enforces equality only at
the identified endpoints, rather than the stronger symmetry
$ B(\bar x,\theta)= B(\bar x,-\theta)$.
Figure~\ref{fig:pwbb_wrap_consistency} illustrates this issue
for the \gls{CT} model described in
Section~\ref{sec:experiments}. Without periodic-boundary
consistency, the barrier takes opposite signs for \(38.66\%\)
of \(10{,}000\) paired states representing the same physical
state at \(\theta=-\pi\) and \(\theta=\pi\). Enforcing
\eqref{eq:pwbb_wrap_C0} eliminates these sign inconsistencies.

\paragraph{Alternating synthesis}
\gls{pwbb} follows the alternating procedure of~\secref{sec:algorithm_v2}, with the lifted quantities and
conditions replaced by their counterparts in the \gls{pwbb}
formulation~\eqref{eq:pwbb_sos_conditions}. The growth condition
\eqref{eq:growth} is applied in the original coordinates to preserve the previous certified set,
and the equivalent linear constraints~\eqref{eq:pwbb_wrap_coeff} are imposed in every
barrier optimization step, including initialization.

\section{Sampled-Data Safety Filter for JBS}
\label{sec:zoh_safety_filter}

The \gls{bbs} methods return a barrier \(B\), a control witness \(u_B\), and
a multiplier \(\lambda\) certifying continuous-time invariance on
\(\mathcal C\). Deploying this certificate through a sampled-data
\gls{cbf}-\gls{qp} raises two issues. First, \(\lambda\) is unrestricted in
sign, so wherever \(\lambda(x)<0\) and \(B(x)>0\) the certified condition
requires the barrier to increase, which is needlessly restrictive at
runtime; we therefore replace \(\lambda\) by a fixed rate \(\gamma\).
Second, enforcing the barrier condition only at the sampling instants does
not prevent \(B\) from becoming negative between them, since the input is
frozen under \gls{zoh}.

We address the first by certifying an upper bound \(\bar\lambda\) on
\(\lambda\), which yields a lower bound on the admissible \(\gamma\), and
the second by tightening the runtime constraint with the local second-order
margin of~\cite{9417092}, which yields an upper bound in terms of \(T_s\).
The resulting filter keeps \(B(x(t))\geq0\) throughout each hold whenever
the \gls{qp} is feasible at the sampling instant. 
For a given \(T_s\), the
bound \(\bar\lambda\) is computed offline and fixes the interval from which
\(\gamma\) is selected; the margin \(\nu_k\) is computed at each instant.

\subsection{Runtime Filter and Fixed Rate}
\label{sec:zoh_filter}

For a held input \(u\in\mathcal U\), let
\(f_u(x):=f(x)+g(x)u\) denote the corresponding vector field
from~\eqref{eq:dynamics}. Define the fixed-rate barrier expression as
\begin{equation}
\label{eq:zoh_invariance_expression}
    r_\gamma(x,u)
    :=\nabla B(x)^\top f_u(x)+\gamma B(x),
\end{equation}
with \(\gamma>0\). At \(t_k=kT_s\), with \(x_k:=x(t_k)\), the filter solves
\begin{equation}
\label{eq:zoh_qp}
\begin{aligned}
    u_k^\star\in\arg\min_{u\in\mathcal U}
    &\quad \tfrac{1}{2}\lVert u-u_{\mathrm{nom},k}\rVert^2 \\
    \mathrm{s.t.}
    &\quad r_\gamma(x_k,u)\geq\nu_k,
\end{aligned}
\end{equation}
where \(\nu_k\geq0\) is the inter-sample margin defined
in~\eqref{eq:zoh_nu}. A minimizer exists whenever the feasible set is
nonempty, and is held constant on \([t_k,t_{k+1})\). The witness \(u_B\) is
never applied; it is used only to derive the rate bound below.

By Proposition~\ref{prop:trajectory_consistency}, the invariance
condition~\eqref{eq:sos_lifted_invariance} evaluated at \(z=\phi(x)\) gives
\(\nabla B^\top f_{u_B}+\lambda B\geq0\) on \(\mathcal C\) in the original
coordinates~\eqref{eq:pullback_function}. Adding and subtracting
\(\lambda(x)B(x)\) in~\eqref{eq:zoh_invariance_expression},
% Evaluating the invariance condition~\eqref{eq:sos_lifted_invariance} at
% \(z=\phi(x)\) and using the recovery map~\eqref{eq:pullback_function} gives,
% by Proposition~\ref{prop:trajectory_consistency},
% \(\nabla B^\top f_{u_B}+\lambda B\geq0\) on \(\mathcal C\). Adding and
% subtracting \(\lambda(x)B(x)\) in~\eqref{eq:zoh_invariance_expression},
\begin{equation}
\label{eq:gamma_bridge}
    r_\gamma\bigl(x,u_B(x)\bigr)\geq\bigl(\gamma-\lambda(x)\bigr)B(x),
    \qquad x\in\mathcal C .
\end{equation}
Both \(\lambda\) and \(\gamma\) prescribe an admissible decay rate for
\(B\), so the smaller rate is the stricter requirement: if
\(\gamma<\lambda(x)\) where \(B(x)>0\), the filter asks for slower decay
than was certified during synthesis, and the witness need not satisfy the runtime
constraint. Imposing
\begin{equation}
\label{eq:gamma_lower}
    \gamma\geq\bar\lambda\geq\lambda(x),
    \qquad x\in\mathcal C,
\end{equation}
therefore gives \(r_\gamma(x,u_B(x))\geq0\) on \(\mathcal C\): the witness
satisfies the untightened fixed-rate condition throughout the certified set.

\subsection{Local Second-Order Inter-Sample Margin}
\label{sec:zoh_margin}

Since \(\dot u=0\) under \gls{zoh}, differentiating \(B\) twice along the
original dynamics gives
\begin{equation}
\label{eq:zoh_Bddot}
    \ddot B(x,u)=
    f_u(x)^\top\nabla^2B(x)f_u(x)
    +\nabla B(x)^\top\frac{\partial f_u}{\partial x}(x)f_u(x).
\end{equation}
Only downward curvature can reduce \(B\). Let \(\mathcal R_k^{+}\)
overapproximate the states reachable from \(x_k\) over \([0,T_s]\) under any
constant \(u\in\mathcal U\), and let \(\eta_k\geq0\) satisfy
\begin{equation}
\label{eq:zoh_eta}
    \eta_k\geq
    \sup_{\xi\in\mathcal R_k^{+},\,u\in\mathcal U}
    \bigl[-\ddot B(\xi,u)\bigr]_+ .
\end{equation}
All inputs are covered because \(\nu_k\) is computed
before~\eqref{eq:zoh_qp} selects \(u_k^\star\). Evaluating the bound locally
on \(\mathcal R_k^{+}\) rather than uniformly over \(\mathcal C\) reduces the
tightening.

Following~\cite[Th.~3]{9417092}, let \(x_k\in\mathcal C\) and let
\(u_k\in\mathcal U\) be feasible for~\eqref{eq:zoh_qp}. Integrating
\(\ddot B\geq-\eta_k\) twice and using
\(\dot B(x_k,u_k)\geq\nu_k-\gamma B(x_k)\), which is the constraint
in~\eqref{eq:zoh_qp}, gives for \(\tau\in[0,T_s]\)
\begin{equation}
\label{eq:zoh_raw_bound}
    B\bigl(x(t_k+\tau)\bigr)\geq
    (1-\gamma\tau)B(x_k)+\nu_k\tau-\frac{\eta_k}{2}\tau^2 .
\end{equation}
The last two terms are nonnegative for all \(\tau\in[0,T_s]\) exactly when
\(\nu_k\geq\tau\eta_k/2\) on that interval, and the smallest such margin is obtained at
\(\tau=T_s\), 
yielding
\begin{equation}
\label{eq:zoh_nu}
    \nu_k:=\frac{T_s}{2}\eta_k .
\end{equation}
With this choice,
\begin{equation}
\label{eq:zoh_inter_sample_bound}
\begin{aligned}
B\bigl(x(t_k+\tau)\bigr)
&\geq (1-\gamma\tau)B(x_k) \\
&\quad+\frac{\eta_k}{2}\tau(T_s-\tau).
\end{aligned}
\end{equation}
% This choice is sufficient but not necessary, since the first term already
% supplies margin when \(B(x_k)>0\). 
The second term is nonnegative on
\([0,T_s]\), and the first is nonnegative when \(\gamma T_s\leq1\). Hence
\(B(x(t))\geq0\) throughout the hold, and
applying the same argument at successive instants preserves
\(B(x(t))\geq0\) as long as~\eqref{eq:zoh_qp} remains feasible; recursive
feasibility is discussed in Remark~\ref{rem:zoh_recursive}.

Combining with~\eqref{eq:gamma_lower} gives the admissible rate interval
\begin{equation}
\label{eq:gamma_interval}
    \bar\lambda\leq\gamma\leq\frac{1}{T_s},
\end{equation}
nonempty when \(T_s\bar\lambda\leq1\). The two bounds have opposite origins:
the lower keeps the synthesized certificate admissible at runtime, the upper
limits how much of the barrier one hold may consume. Within the interval, a
larger \(\gamma\) relaxes the constraint in the interior of \(\mathcal C\)
and therefore intervenes less.

\subsection{Certification of the Multiplier Bound}
\label{sec:zoh_certification}

The filter is common to both representations, which differ only in the set
over which \(\bar\lambda\) is certified; \eqref{eq:zoh_Bddot} is evaluated
on the original dynamics in both cases. For \gls{ldbb}, the bound is
computed in lifted coordinates,
\begin{equation}
\label{eq:ldbb_lambda_bound}
    \bar\lambda:=\min\bigl\{\ell\in\mathbb R:\;
    \ell-\hat\lambda(z)
    \geq_{\hat{\mathcal C}\cap\hat{\mathcal K}^{+}}0\bigr\},
\end{equation}
and evaluating at \(z=\phi(x)\) gives \(\lambda(x)\leq\bar\lambda\) on
\(\mathcal C\). For \gls{pwbb}, the sector-wise
condition~\eqref{eq:pwbb_C3} makes the same argument valid in every sector,
so
\begin{equation}
\label{eq:pwbb_lambda_bound}
    \bar\lambda_{\mathrm{PW}}:=\min\bigl\{\ell\in\mathbb R:\;
    \ell-\lambda_q(x)\geq_{\mathcal C\cap\mathcal K_q}0,\;
    q=1,\ldots,Q\bigr\},
\end{equation}
and~\eqref{eq:gamma_interval} applies with \(\bar\lambda_{\mathrm{PW}}\).

\begin{remark}[Recursive feasibility]
\label{rem:zoh_recursive}
The guarantee above holds at each instant at which~\eqref{eq:zoh_qp} is
feasible. Recursive feasibility would follow if a certified witness margin
\(\underline m_\gamma\leq r_\gamma(x,u_B(x))\) on
\(\mathcal C\cap\mathcal S\) satisfied \(\underline m_\gamma\geq\nu_k\),
since \(u_B(x_k)\) would then be a feasible candidate at every certified
state. Obtaining a tight \(\underline m_\gamma\) is not straightforward: the
boundary \(\{B=0\}\) lies in \(\mathcal C\cap\mathcal S\), and there
\(\gamma B\) vanishes, so \(\underline m_\gamma\) is limited by the inward
flow the witness provides on that boundary, independently of \(\gamma\).
Enlarging that reserve during synthesis makes the condition attainable but
shrinks the certified set, so recursive feasibility and certified coverage
trade off against each other.
% ; selecting the pair \((\gamma,T_s)\) from this
% comparison is left for future work.
\end{remark}
\section{Successive-Barrier Synthesis}
\label{sec:successive_barrier_synthesis}

We additionally consider the sample-guided \gls{sbs}
method of~\cite{Wajid_successive}. Given a finite set of physical
synthesis samples and a finite set of admissible controls
\(\mathcal U_{\mathrm{fin}}\subset\mathcal U\), the method
successively constructs polynomial barriers associated with
fixed controls and removes samples covered by each accepted
barrier. We use the nonnegative-superlevel convention
in~\eqref{eq:candidate_set}, reversing the barrier sign used
in~\cite{Wajid_successive}. Its implementation with
the piecewise-polynomial dynamics~\eqref{eq:pw_prelim_dynamics}
is denoted by \gls{pwsb}.

To obtain \gls{ldsb}, we retain the same sample-guided search,
fixed-control candidates, and candidate-selection procedure through transit-time feasibility,
but replace the piecewise-polynomial dynamics with the exact
lifted dynamics~\eqref{eq:lifted_dynamics}. Let
\(\hat{\mathcal B}_{<\ell}\) denote the barriers accepted before
level \(\ell\). The level-\(\ell\) derivative condition is
enforced on
\[
    \hat{\Omega}_{\ell}
    :=
    \hat{\mathcal K}
    \cap
    \bigcap_{\hat B\in\hat{\mathcal B}_{<\ell}}
    \{z\mid \hat B(z)\leq0\},
\]
as
\begin{equation}
\label{eq:ldsb_invariance}
    \nabla\hat B_{\ell,r}(z)^\top
    \bigl(\hat f(z)+\hat g(z)u_{\ell,r}\bigr)
    +\lambda_{\mathrm{SB}}\hat B_{\ell,r}(z)
    \geq_{\hat{\Omega}_{\ell}}0,
\end{equation}
where \(u_{\ell,r}\in\mathcal U_{\mathrm{fin}}\) and
\(\lambda_{\mathrm{SB}}>0\) are fixed during each candidate
synthesis. Unsafe-set separation is imposed as
in~\eqref{eq:sos_lifted_safety}, with the lifting-induced
equalities enforced through unrestricted polynomial
multipliers. Thus, \gls{ldsb} replaces the sector-wise
approximation-error conditions of \gls{pwsb} with certification
on the lifting-induced algebraic set.

The remaining sample-removal and transit-time feasibility tests
are unchanged from~\cite{Wajid_successive}; physical
samples are evaluated in \gls{ldsb} through
\(z=\phi(x)\).
\section{Experiments}
\label{sec:experiments}

We evaluate exact polynomial lifting against piecewise-polynomial
approximation on the \gls{CT} and \gls{PM} systems, within both
synthesis methods introduced above (\gls{pwbb} versus \gls{ldbb}, and
\gls{pwsb} versus \gls{ldsb}), to investigate three questions:
(i) how lifting affects certified coverage and computational cost
relative to piecewise approximation, and how the sector resolution of
\gls{pwbb} affects this comparison;
(ii) how the \gls{pwbb}- and \gls{ldbb}-based safety filters compare in
trajectory safety, intervention, task performance, and online
computation; and
(iii) whether \gls{ldsb} improves coverage and computational cost
relative to \gls{pwsb}.

For~\gls{bbs}, we additionally evaluate
closed-loop safety on a target-tracking task, with all
trajectories simulated using the original non-polynomial
dynamics.

All computations were performed on a workstation running
Ubuntu 24.04.1 LTS with an AMD Ryzen Threadripper PRO
7965WX processor and 256 GB of RAM, using Julia 1.12.5;
MOSEK was used for SOS optimization with each SDP solve
limited to two threads.

\paragraph{Evaluation method}
Comparisons are matched within each synthesis method.
For \gls{pwbb} and \gls{ldbb}, we use the same certification
domain, unsafe set, evaluation samples, and polynomial degrees
($\deg B=4$, $\deg u_B=2$) for each system.
All reported
certificates are independently re-verified after synthesis as
described in~\appref{app:post_synthesis_verification}.

Estimated certified coverage is evaluated on $N=10{,}000$ uniformly
sampled physical states, with lifted certificates evaluated
through~\eqref{eq:pullback_function}; the metric is defined
in~\appref{app:evaluation_metrics}.
For the \gls{bbs} comparison,
synthesis is run for $100$ iterations for \gls{CT} and
$250$ for \gls{PM}, terminating earlier if a
subproblem becomes infeasible (Algorithm~\ref{alg:alternation_v2}). Empirical convergence is defined as the
first iteration whose estimated coverage is within one
percentage point of the maximum observed over the complete
run.

For \gls{pwbb} and \gls{ldbb} comparison, we simulate a target-tracking task to evaluate safety and
filtering performance. For filtering evaluation, minimal intervention is measured by
$J_{\mathrm{int},2}$, corresponding to the objective in~\eqref{eq:zoh_qp}.
Goal-tracking performance
and task completion time are measured by $J_g$ and
$T_{\mathrm{reach}}$, respectively, with
$T_{\mathrm{reach}}$ evaluated over successful trajectories for
each method. Conservativeness around the obstacle is assessed by the
minimum obstacle clearance \(d_{\min}\) attained along the
trajectory.
The metric definitions are provided in
~\appref{app:evaluation_metrics}.

For \gls{ldsb}, we compare against the \gls{pwsb} implementation provided with~\cite{Wajid_successive} on its non-polynomial \gls{CT} and \gls{PM} benchmarks,
using the same domains and polynomial degree.
The comparison is done on different sampling sizes to study the effect on barrier-bank size, coverage, synthesis time, and
memory use; further details are given in the supplementary
material.

\subsection{System Dynamics and Case Studies}
\label{sec:case_studies}

We consider two robotic systems with trigonometric dynamics:
the \gls{CT} model and a \gls{PM}. 
For each system,
we start from the original non-polynomial dynamics~\eqref{eq:dynamics}
and construct the two model representations introduced in
~\secref{sec:polynomial_representations}. The
piecewise-polynomial representation~\eqref{eq:pw_prelim_dynamics} 
is defined over the sectors 
$\mathcal I_q$, whereas the exact lifted representation~\eqref{eq:lifted_dynamics} 
satisfies the lifting-induced algebraic constraints in
\eqref{eq:lifting_relations}. 

For both systems, let
$p=[p_x,p_y]^\top$ denote the position,
$\theta$ the orientation, $\omega$ the angular velocity, and
$u=[u_1,u_2]^\top$ the control input. Both systems must avoid
the circular unsafe set
\begin{equation}
\label{eq:case_study_unsafe_set}
    \mathcal X_u
    :=
    \left\{
        x
        \;\middle|\;
        r_o^2-p_x^2-p_y^2\geq 0
    \right\},
    \qquad
    r_o=0.1 .
\end{equation}

\paragraph{\gls{CT} model.}
The \gls{CT} state is
$
    x
    =
    [p_x,p_y,v,\theta,\omega]^\top,
$
where $v$ is the forward speed. Its dynamics are
\begin{equation}
\label{eq:ct_original_dynamics}
    % \Psi_{\mathrm{CT}}\colon\quad
    \dot x
    =
    \begin{bmatrix}
        v\cos\theta &
        v\sin\theta &
        u_1 &
        \omega &
        u_2
    \end{bmatrix}^{\top}.
\end{equation}
where $u\in[-1,1]^2$.
\paragraph{\gls{PM} model.}
The \gls{PM} state is
$
    x
    =
    [p_x,p_y,v_x,v_y,\theta,\omega]^\top,
$
where $[v_x,v_y]^\top$ is the translational velocity. Its
dynamics are
\begin{equation}
\label{eq:pm_original_dynamics}
% \Psi_{\mathrm{PM}}\colon\quad
\begin{aligned}
    &\dot p_x = v_x, \ 
    \dot p_y = v_y, \ 
    \dot v_x = (u_1+u_2)\sin\theta, \\
    &\dot v_y = (u_1+u_2)\cos\theta-g_0, \ 
    \dot\theta = \omega, \ 
    \dot\omega = u_1-u_2,
\end{aligned}
\end{equation}
where $g_0=2$ is the normalized gravitational acceleration and $u\in[-2,2]^2$.
In contrast to the \gls{CT} model, the \gls{PM} contains
input-dependent trigonometric terms. The corresponding
piecewise-polynomial and exact lifted representations are provided in~\appref{app:experimental_models}.

\subsection{Joint Barrier Synthesis Results}
\paragraph{Synthesis settings}
For these experiments, the \gls{CT}
certification domain $\mathcal D$ in the original
coordinates is
\[
    p_x,p_y,v\in[-2,2],\qquad
    \theta\in[-\pi,\pi],\qquad
    \omega\in[-1,1].
\]
For \gls{PM},
% the certification domain
$\mathcal D$ is
\[
    p_x,p_y\in[-2,2],\qquad
    v_x,v_y\in[-10,10],
\]
\[    \theta\in[-\pi,\pi],\qquad
    \omega\in[-10,10].
\]
The barrier scale is normalized at a common reference point so that
values are comparable across methods.

\paragraph{Synthesis Results}
As shown in \figref{fig:bb_synthesis_convergence}, increasing
the \gls{pwbb} sector resolution from $Q=4$ to $Q=16$
raises the coverage 
at empirical convergence from $45.3\%$
to $51.1\%$ for \gls{CT} and from $49.4\%$ to $51.9\%$
for \gls{PM}.
Over the same range, convergence shifts from
$39$ to $62$ iterations and from $6.1$ to
$32.8\,\mathrm{min}$ for \gls{CT}, and from $80$ to
$117$ iterations and from $35.4$ to
$170.3\,\mathrm{min}$ for \gls{PM}. 
Relative to \gls{pwbb} at Q=16, \gls{ldbb} reaches $55.9\%$ coverage at
iteration $29$ after $5.9\,\mathrm{min}$ for \gls{CT}, and $56.0\%$ at
iteration $161$ after $114.4\,\mathrm{min}$ for \gls{PM}. On \gls{PM},
\gls{ldbb} therefore requires more iterations than \gls{pwbb} at $Q=16$
but roughly half the time per iteration
($0.71$ versus $1.46\,\mathrm{min}$), consistent with the reduced
\gls{psd}-block count reported below.

\begin{figure*}[t]
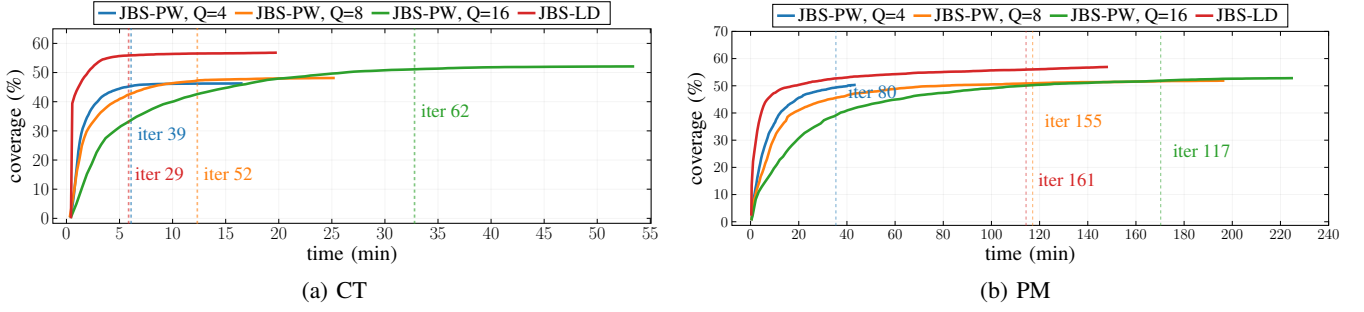

    \centering
    \begin{minipage}[b]{0.49\textwidth}
        \centering
        \resizebox{\linewidth}{!}{%
        \input{figures/experiment/coordturn/CT_synthesis_convergence.tex}%
        }
        \small (a) \gls{CT}
    \end{minipage}
    % \hfill
    \begin{minipage}[b]{0.49\textwidth}
        \centering
        \resizebox{\linewidth}{!}{%
        \input{figures/experiment/planar_multirotor/PM_synthesis_convergence}%
        }
        \small (b) \gls{PM}
    \end{minipage}
        \caption{Estimated certified coverage versus cumulative synthesis time for
\gls{pwbb} at $Q\in\{4,8,16\}$ and \gls{ldbb}. Dashed vertical lines mark
empirical convergence, labeled with the iteration number. Curves end at
the iteration limit or on an infeasible subproblem (see Algo.~\ref{alg:alternation_v2}).} 
    \label{fig:bb_synthesis_convergence}
\end{figure*}

\begin{figure*}[t]
    \centering
    \begin{minipage}[b]{0.24\textwidth}
        \centering
            \includegraphics[width=\linewidth]{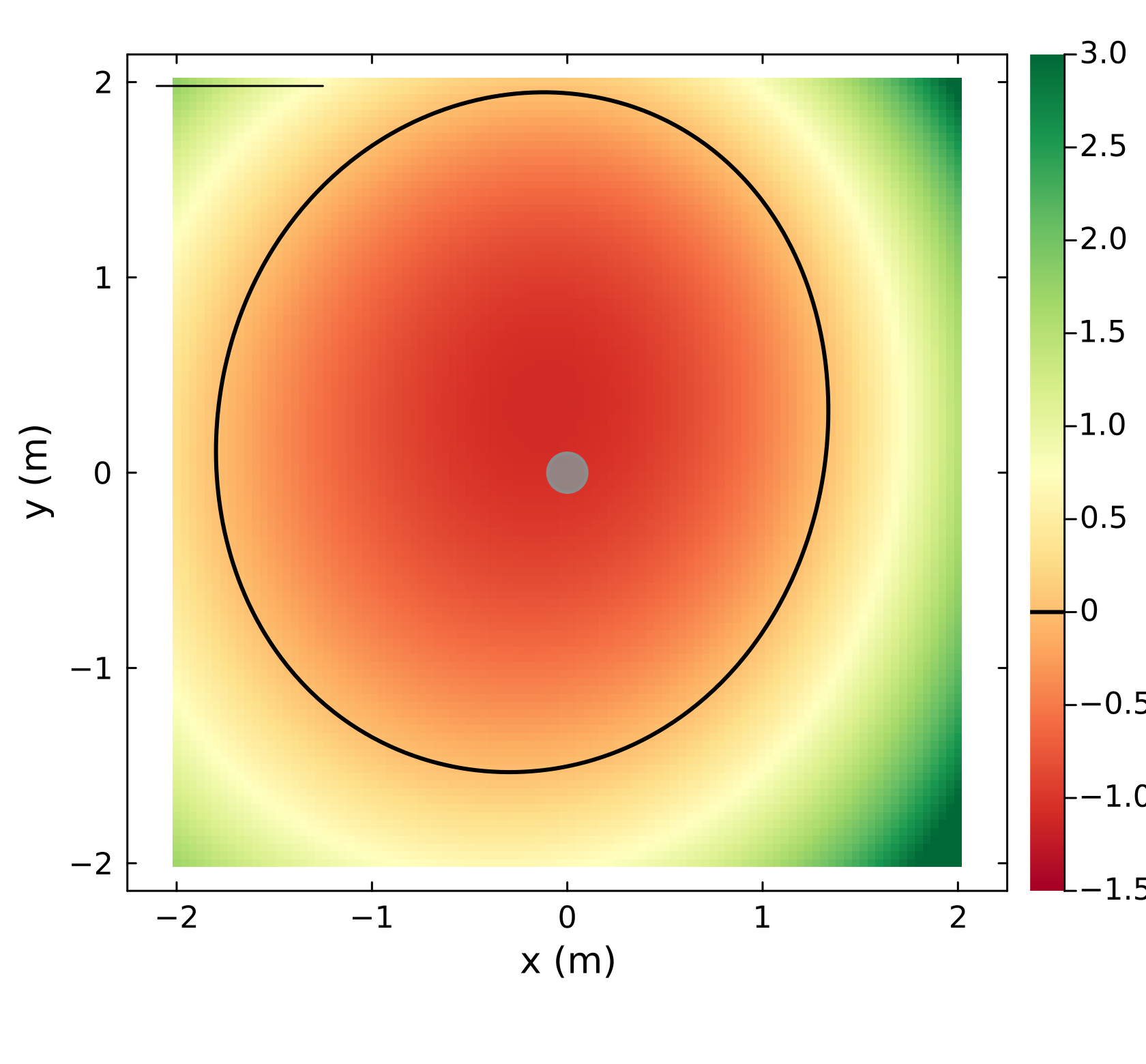}
        \small (a) \gls{CT}: \gls{pwbb} 
    \end{minipage}
    \hfill
    \begin{minipage}[b]{0.24\textwidth}
        \centering
        \includegraphics[width=\linewidth]{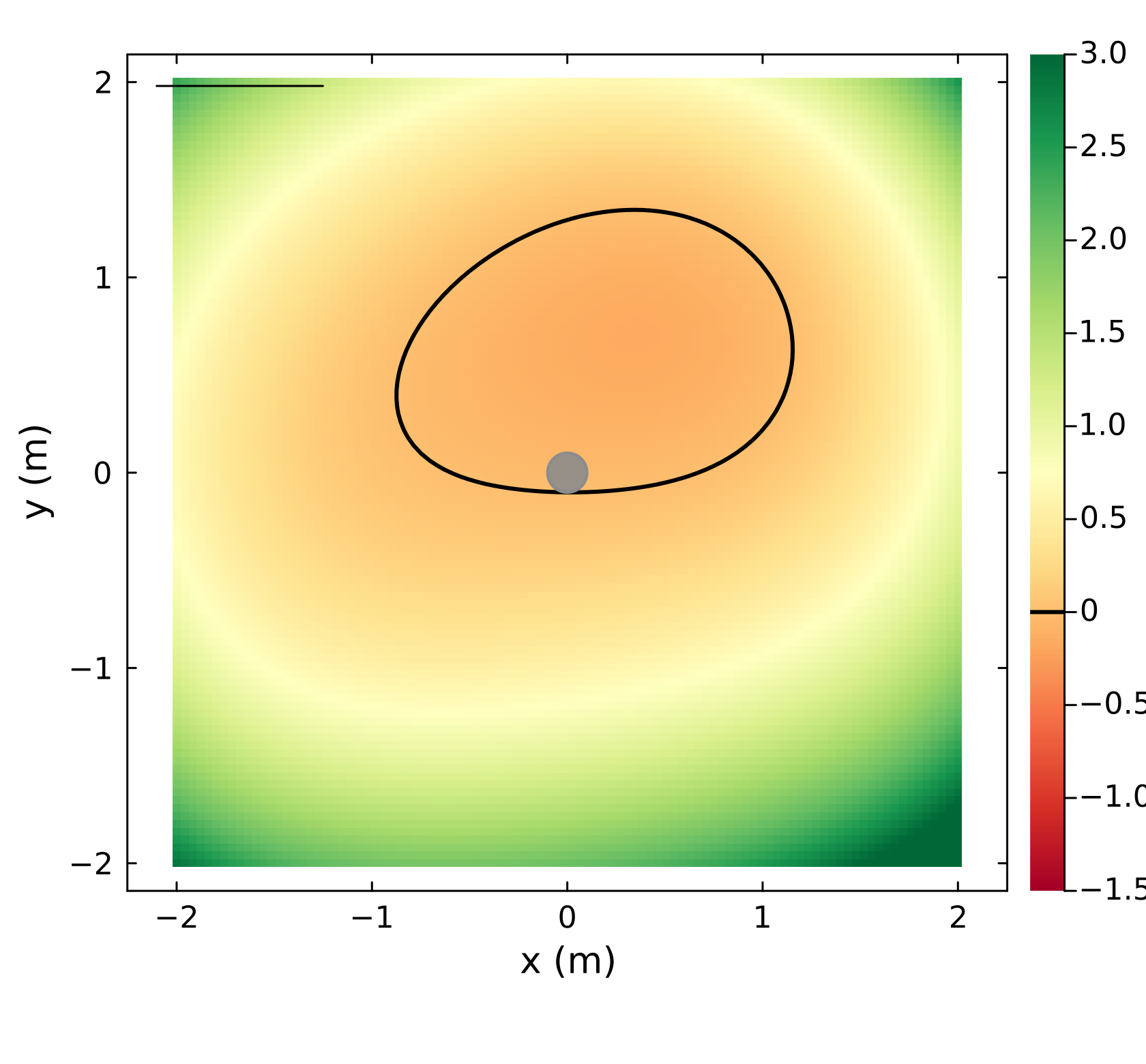}
        \small (b) \gls{CT}: \gls{ldbb}
    \end{minipage}
    \hfill
    \begin{minipage}[b]{0.24\textwidth}
        \centering
        \includegraphics[width=\linewidth]{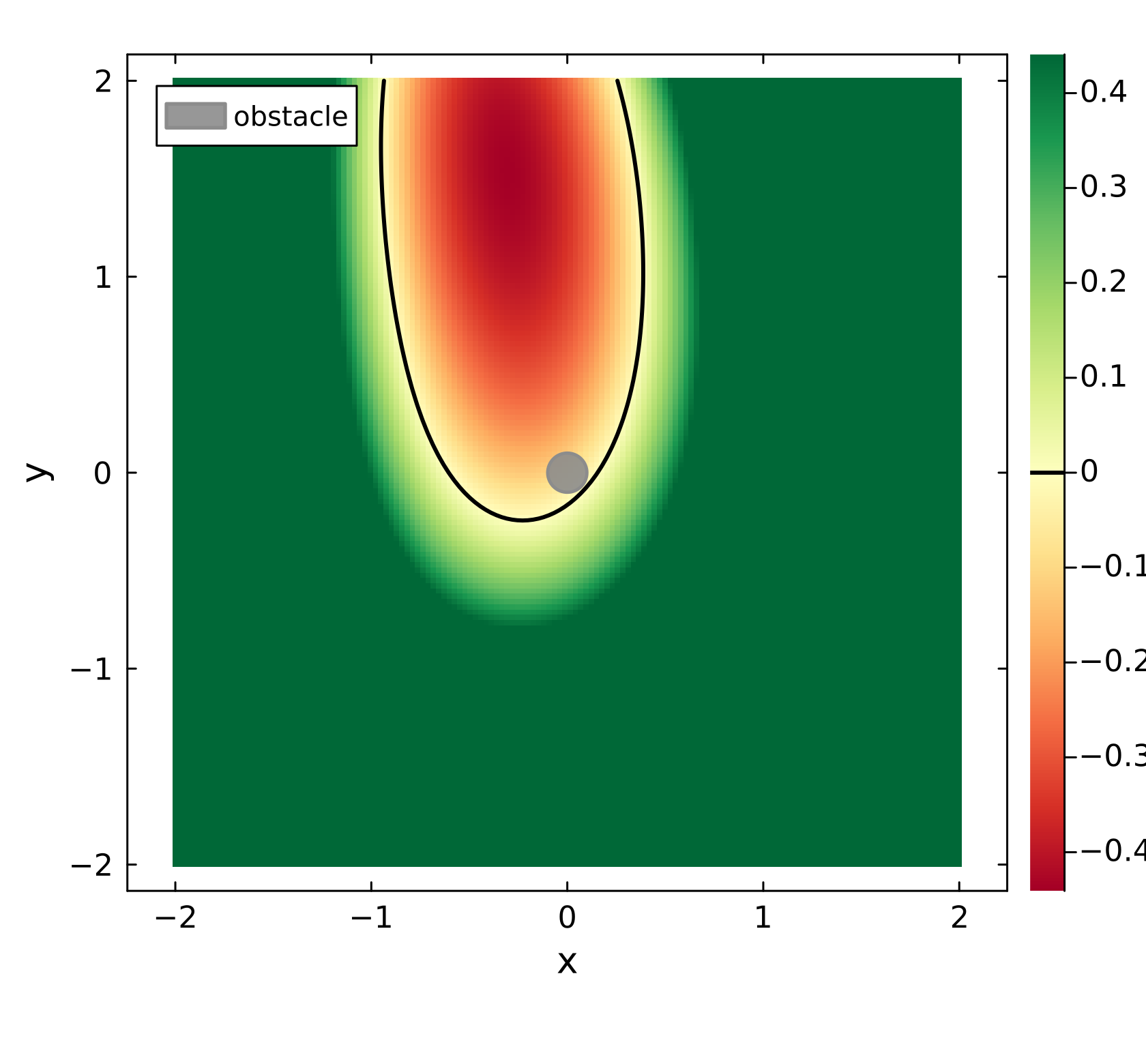}
        \small (c) \gls{PM}: \gls{pwbb}
    \end{minipage}
    \hfill
    \begin{minipage}[b]{0.24\textwidth}
        \centering
        \includegraphics[width=\linewidth]{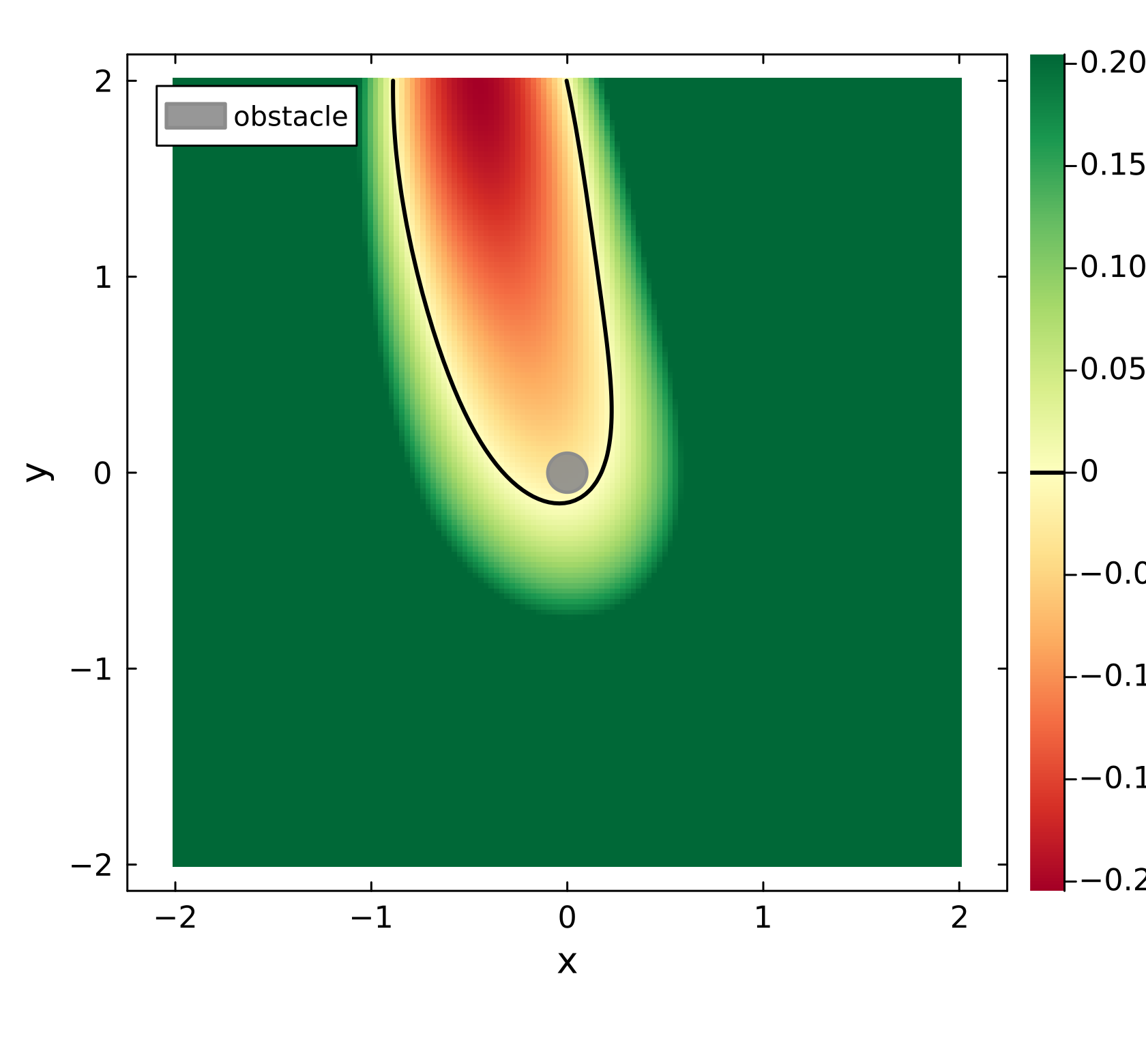}
        \small (d) \gls{PM}: \gls{ldbb}
    \end{minipage}

    % \caption{Representative barrier slices for \gls{pwbb} and
    % \gls{ldbb}. The gray region denotes the obstacle and the
    % black curve denotes $B(x)=0$; the certified region satisfies
    % $B(x)\geq0$.}
    \caption{Representative barrier slices for \gls{pwbb} (Q = 16) and
\gls{ldbb}, heatmap of barrier function sliced at $\theta=\pi$ with $v=0.3,\omega=0$ for
CT and $v_x=0.3,v_y=\omega=0$ for PM. The gray region denotes the
obstacle and the black curve denotes $B(x)=0$; the certified region
satisfies $B(x)\geq0$.}
    \label{fig:bb_slices}
\end{figure*}

In comparison with the finest tested \gls{pwbb} resolution ($Q =16$), \gls{ldbb} achieves higher coverage with lower synthesis time on both systems.
To complement these 
results, \figref{fig:bb_slices} compares representative
$x$-$y$ slices of the synthesized barriers. \gls{ldbb} shows a larger certified area than \gls{pwbb}.

Peak process memory also increases with \gls{pwbb} sector
resolution, from $2.32$ to $3.36\,\mathrm{GiB}$ for \gls{CT}
and from $4.58$ to $7.79\,\mathrm{GiB}$ for \gls{PM} as
$Q$ increases from $4$ to $16$. The corresponding
\gls{ldbb} values are $3.23\,\mathrm{GiB}$ and
$4.49\,\mathrm{GiB}$, respectively.

To compare the resulting \gls{sdp} structures, we consider the
barrier step at the synthesis polynomial degrees. \gls{ldbb}
increases the state dimension from \(5\) to \(6\) for \gls{CT}
and from \(6\) to \(7\) for \gls{PM}, increasing the degree-\(4\)
barrier basis from \(126\) to \(210\) and from \(210\) to \(330\)
monomials, respectively. In contrast, \gls{pwbb} certifies four
approximation-error vertices per sector, giving \(16\) and \(64\)
invariance conditions for \(Q=4\) and \(Q=16\), whereas
\gls{ldbb} uses a single algebraically constrained condition.
The resulting \gls{psd}-block count increases from \(402\) at
\(Q=4\) to \(1314\) at \(Q=16\) for \gls{CT}, and from \(468\)
to \(1524\) for \gls{PM}. In comparison, \gls{ldbb} requires
\(135\) and \(156\) blocks for \gls{CT} and \gls{PM},
respectively.

% \FloatBarrier

\paragraph{Closed-Loop Results}
Here we evaluate the
target-tracking safety filter of
~\secref{sec:zoh_safety_filter}.
The filter sampling time is  
$T_s=0.05\,\mathrm{s}$, and the computed input is held
constant between updates. The original nonlinear dynamics are
simulated with RK4 using a $0.01\,\mathrm{s}$ step.
% For the \gls{CT} and \gls{PM}, a 
A trajectory terminates upon reaching
the goal or after $30\,\mathrm{s}$  for \gls{CT} and $10\,\mathrm{s}$ for \gls{PM}. Goal reaching requires
$\lVert p-p_{\mathrm g}\rVert_2\leq0.05\,\mathrm{m}$, $|\omega|\leq0.1\,\mathrm{rad\,s^{-1}}$
with speed $|v| \leq 0.1\,\mathrm{m\,s^{-1}}$ for \gls{CT} and for $|v_x|  \leq 0.1\,\mathrm{m\,s^{-1}}, |v_y| \leq 0.1\,\mathrm{m\,s^{-1}}$ for \gls{PM}.

For comparison between~\gls{ldbb} and~\gls{pwbb} ($Q = 16$), we generate $500$ start-goal
pairs uniformly over the full $2\pi$ range of approach
directions, with starts and goals placed on opposite sides of
the obstacle such that the nominal trajectories intersect it (see~\figref{fig:bb_rollouts}).
From $500$ candidate
start points, $500$ and $443$ are common certified points for the \gls{CT} and \gls{PM} models, respectively. Both methods are evaluated on common trajectories with the same starts, goals, and nominal controller described
in~\appref{app:nominal_control}; for \gls{PM}, this controller uses a cascaded position--attitude structure.
% , and all 
All retained trajectory pairs
are used for the quantitative comparison. For visual
clarity, \figref{fig:bb_rollouts} shows only six representative
pairs, with the corresponding nominal trajectories.

For each synthesized certificate, we first compute the \gls{sos}-certified
upper bound \(\bar\lambda\) on the synthesis multiplier
using~\eqref{eq:ldbb_lambda_bound} for \gls{ldbb}
and~\eqref{eq:pwbb_lambda_bound} for \gls{pwbb}.

By~\eqref{eq:gamma_lower},
this value is the lower bound on the admissible runtime rate, giving
$2.7968/0.6529\,\mathrm{s}^{-1}$ for
\gls{pwbb}/\gls{ldbb} on \gls{CT} and
$1.2265/3.0615\,\mathrm{s}^{-1}$ on \gls{PM}.
For minimal-intervention controller design, we choose
\(\gamma=\gamma_{\max}=19.98\,\mathrm{s}^{-1}\), above these
lower bounds and slightly below the sampled-data upper bound
\(1/T_s=20\,\mathrm{s}^{-1}\) to retain a small numerical margin. A larger \(\gamma\) relaxes the barrier condition
in the interior of the certified set, allowing greater freedom
to follow the nominal controller before approaching the
boundary.

For the simulation, \(\eta_k\) in~\eqref{eq:zoh_eta} is evaluated locally by deterministic global search with multistart refinement and a small numerical inflation; the resulting bound is numerical rather than certified.
For qualitative comparison, trajectories using the synthesized
\(\lambda(x)\) and the fixed rate \(\gamma_{\max}\) are shown
in \figref{fig:bb_rollouts}. Both controllers avoid the obstacle, but the \(\gamma_{\max}\)-based
controller is less restrictive and passes closer to the obstacle boundary.
The \(\lambda(x)\)-based controller is more conservative because a smaller
rate is the stricter requirement, and wherever \(\lambda(x)<0\) the
condition becomes \(\dot B\geq|\lambda(x)|B>0\), forcing the barrier to
increase.

% ============================================================
% Four-panel trajectory figure assembler
%
% Main-document preamble requirements:
%   \usepackage{graphicx}
%   \usepackage{pgfplots}
%   \usepackage{pgfplotstable}
%   \usetikzlibrary{backgrounds}
%   \pgfplotsset{compat=1.18}
% ============================================================

% trajectory_template_fixed2.tex
% Required: \usepackage{pgfplots}
%           \usepackage{pgfplotstable}
%           \usetikzlibrary{backgrounds}
% Filled five-pointed star (pentagram) plot mark.
% Geometry matches \ngram{R}{5}{...}: inner radius = R*sin(90-72)/sin(90+36)
%   = R*sin(18)/sin(126) = 0.381966*R.
% pgfplots' built-in mark=star is an asterisk and ignores fill.
\pgfdeclareplotmark{ngramstar}{%
  \pgfpathmoveto{\pgfqpoint{0pt}{\pgfplotmarksize}}%
  \pgfpathlineto{\pgfqpoint{0.224514\pgfplotmarksize}{0.309017\pgfplotmarksize}}%
  \pgfpathlineto{\pgfqpoint{0.951057\pgfplotmarksize}{0.309017\pgfplotmarksize}}%
  \pgfpathlineto{\pgfqpoint{0.363271\pgfplotmarksize}{-0.118034\pgfplotmarksize}}%
  \pgfpathlineto{\pgfqpoint{0.587785\pgfplotmarksize}{-0.809017\pgfplotmarksize}}%
  \pgfpathlineto{\pgfqpoint{0pt}{-0.381966\pgfplotmarksize}}%
  \pgfpathlineto{\pgfqpoint{-0.587785\pgfplotmarksize}{-0.809017\pgfplotmarksize}}%
  \pgfpathlineto{\pgfqpoint{-0.363271\pgfplotmarksize}{-0.118034\pgfplotmarksize}}%
  \pgfpathlineto{\pgfqpoint{-0.951057\pgfplotmarksize}{0.309017\pgfplotmarksize}}%
  \pgfpathlineto{\pgfqpoint{-0.224514\pgfplotmarksize}{0.309017\pgfplotmarksize}}%
  \pgfpathclose%
  \pgfusepathqfillstroke%
}

\definecolor{trajone}{rgb}{0.1216,0.4667,0.7059}
\definecolor{trajtwo}{rgb}{1.0000,0.4980,0.0549}
\definecolor{trajthree}{rgb}{0.1725,0.6275,0.1725}
\definecolor{trajfour}{rgb}{0.8392,0.1529,0.1569}
\definecolor{trajfive}{rgb}{0.5804,0.4039,0.7412}
\definecolor{trajsix}{rgb}{0.5490,0.3373,0.2941}

\pgfplotsset{
trajectory axis/.style={
point meta max={nan},
point meta min={nan},
filter discard warning=false,
legend cell align={left},
legend columns={4},
legend style={color=black,draw opacity=1,line width=1,solid,fill=white,fill opacity=1,text opacity=1,font={\fontsize{25pt}{23.4pt}\selectfont},text=black,cells={anchor=center},at={(0.5,1.02)},anchor=south},
axis background/.style={fill=white,opacity=1},
anchor=north west,
xshift=1mm,
yshift=5mm,
width=150.4mm,
height=169.1mm,
scaled x ticks=false,
xlabel={$x (m)$},
xlabel style={at={(ticklabel cs:0.5)},anchor=near ticklabel,font={\fontsize{20pt}{15.6pt}\selectfont},color=black},
xmajorgrids=true,
xmin=-2.3,
xmax=2.3,
xtick={-2,-1,0,1,2},
xticklabels={{$-2$,$-1$,$0$,$1$,$2$}},
xtick align=inside,
xticklabel style={font={\fontsize{20pt}{23.4pt}\selectfont},color=black},
x grid style={color=black,draw opacity=0.25,line width=0.5,solid},
x axis line style={color=black,draw opacity=1,line width=1,solid},
scaled y ticks=false,
ylabel={$y (m)$},
ylabel style={at={(ticklabel cs:0.5)},anchor=near ticklabel,font={\fontsize{20pt}{15.6pt}\selectfont},color=black},
ymajorgrids=true,
ymin=-2.3,
ymax=2.3,
ytick={-2,-1,0,1,2},
yticklabels={{$-2$,$-1$,$0$,$1$,$2$}},
ytick align=inside,
yticklabel style={font={\fontsize{20pt}{23.4pt}\selectfont},color=black},
y grid style={color=black,draw opacity=0.25,line width=0.5,solid},
y axis line style={color=black,draw opacity=1,line width=1,solid},
colorbar=false
},
gamma trajectory/.style={draw opacity=0.85,line width=2.2,solid},
lambda trajectory/.style={draw opacity=0.75,line width=1.1,dashed},
nominal trajectory/.style={color=black,draw opacity=0.6,line width=1.1,dashdotted}
}

% #1 csv, #2 series_id, #3 rollout_idx, #4 color, #5 style
\newcommand{\AddFilteredTrajectory}[5]{%
\addplot[color=#4,unbounded coords=jump,#5] table[col sep=comma,x expr={((\thisrow{series_id}==#2)&&(\thisrow{rollout_idx}==#3)) ? \thisrow{x} : nan},y expr={((\thisrow{series_id}==#2)&&(\thisrow{rollout_idx}==#3)) ? \thisrow{y} : nan}]{#1};%
}
% \newcommand{\AddFilteredTrajectory}[5]{%
% \addplot[
%     color=#4,
%     unbounded coords=jump,
%     restrict expr to domain={\thisrow{series_id}}{#2:#2},
%     restrict expr to domain={\thisrow{rollout_idx}}{#3:#3},
%     #5
% ]
% table[
%     col sep=comma,
%     x=x,
%     y=y
% ]{#1};%
% }

% #1,#2 start; #3,#4 goal; #5 color
% \newcommand{\AddStartGoalMarkers}[5]{%
% \addplot[only marks,mark=*,mark size=3.75pt,mark options={color=black,fill=#5,line width=0.75pt},forget plot] coordinates {(#1,#2)};%
% \addplot[only marks,mark=star,mark size=5.25pt,mark options={color=black,fill=#5,line width=0.75pt},forget plot] coordinates {(#3,#4)};%
% }
% #1,#2 start; #3,#4 goal; #5 color
% #1,#2 start; #3,#4 goal; #5 color
\newcommand{\AddStartGoalMarkers}[5]{%
\addplot[only marks,mark=*,mark size=5.5pt,mark options={draw=black,fill=#5,line width=1pt},forget plot] coordinates {(#1,#2)};%
\addplot[only marks,mark=ngramstar,mark size=7.75pt,mark options={draw=black,fill=#5,line width=1pt},forget plot] coordinates {(#3,#4)};%
}

\newcommand{\TrajectoryPlot}[1]{%
\begin{tikzpicture}[/tikz/background rectangle/.style={fill=white},show background rectangle]%
\begin{axis}[trajectory axis]%
\addplot[color={rgb,1:red,0.55;green,0.55;blue,0.55},area legend,fill={rgb,1:red,0.55;green,0.55;blue,0.55},fill opacity=0.9,draw opacity=1,line width=1.6,solid,domain=0:360,samples=65] ({0.1*cos(x)},{0.1*sin(x)}) \closedcycle;%
\addlegendentry{obstacle \quad}%
\AddFilteredTrajectory{#1}{0}{1}{trajone}{gamma trajectory}%
\addlegendentry{$\gamma_{max}$ \quad}%
\AddFilteredTrajectory{#1}{0}{2}{trajtwo}{gamma trajectory,forget plot}%
\AddFilteredTrajectory{#1}{0}{3}{trajthree}{gamma trajectory,forget plot}%
\AddFilteredTrajectory{#1}{0}{4}{trajfour}{gamma trajectory,forget plot}%
\AddFilteredTrajectory{#1}{0}{5}{trajfive}{gamma trajectory,forget plot}%
\AddFilteredTrajectory{#1}{0}{6}{trajsix}{gamma trajectory,forget plot}%
\AddFilteredTrajectory{#1}{1}{1}{trajone}{lambda trajectory}%
\addlegendentry{$\lambda(x)$ \quad}%
\AddFilteredTrajectory{#1}{1}{2}{trajtwo}{lambda trajectory,forget plot}%
\AddFilteredTrajectory{#1}{1}{3}{trajthree}{lambda trajectory,forget plot}%
\AddFilteredTrajectory{#1}{1}{4}{trajfour}{lambda trajectory,forget plot}%
\AddFilteredTrajectory{#1}{1}{5}{trajfive}{lambda trajectory,forget plot}%
\AddFilteredTrajectory{#1}{1}{6}{trajsix}{lambda trajectory,forget plot}%
\AddFilteredTrajectory{#1}{2}{1}{black}{nominal trajectory}%
\addlegendentry{$x_{nom}$}%
\AddFilteredTrajectory{#1}{2}{2}{black}{nominal trajectory,forget plot}%
\AddFilteredTrajectory{#1}{2}{3}{black}{nominal trajectory,forget plot}%
\AddFilteredTrajectory{#1}{2}{4}{black}{nominal trajectory,forget plot}%
\AddFilteredTrajectory{#1}{2}{5}{black}{nominal trajectory,forget plot}%
\AddFilteredTrajectory{#1}{2}{6}{black}{nominal trajectory,forget plot}%
\AddStartGoalMarkers{1.8}{0.0}{-1.4400000000000002}{0.0}{trajone}%
\AddStartGoalMarkers{0.9000000000000002}{1.5588457268119895}{-0.7200000000000002}{-1.2470765814495917}{trajtwo}%
\AddStartGoalMarkers{-0.8999999999999996}{1.5588457268119897}{0.7199999999999998}{-1.2470765814495919}{trajthree}%
\AddStartGoalMarkers{-1.8}{2.2043642384652358e-16}{1.4400000000000002}{-1.7634913907721887e-16}{trajfour}%
\AddStartGoalMarkers{-0.9000000000000008}{-1.5588457268119893}{0.7200000000000006}{1.2470765814495914}{trajfive}%
\AddStartGoalMarkers{0.9000000000000002}{-1.5588457268119895}{-0.7200000000000002}{1.2470765814495917}{trajsix}%
\end{axis}%
\end{tikzpicture}%
}

% Data files. Change only these four paths if the CSV names change.
\newcommand{\CTPWBBTrajectoryData}{figures/experiment/coordturn/csv_data/pwbb_xy.csv}
\newcommand{\CTLDBBTrajectoryData}{figures/experiment/coordturn/csv_data/ldbb_xy.csv}
\newcommand{\PMPWBBTrajectoryData}{figures/experiment/planar_multirotor/csv_data/pwbb_xy.csv}
\newcommand{\PMLDBBTrajectoryData}{figures/experiment/planar_multirotor/csv_data/ldbb_xy.csv}

\begin{figure}[t]
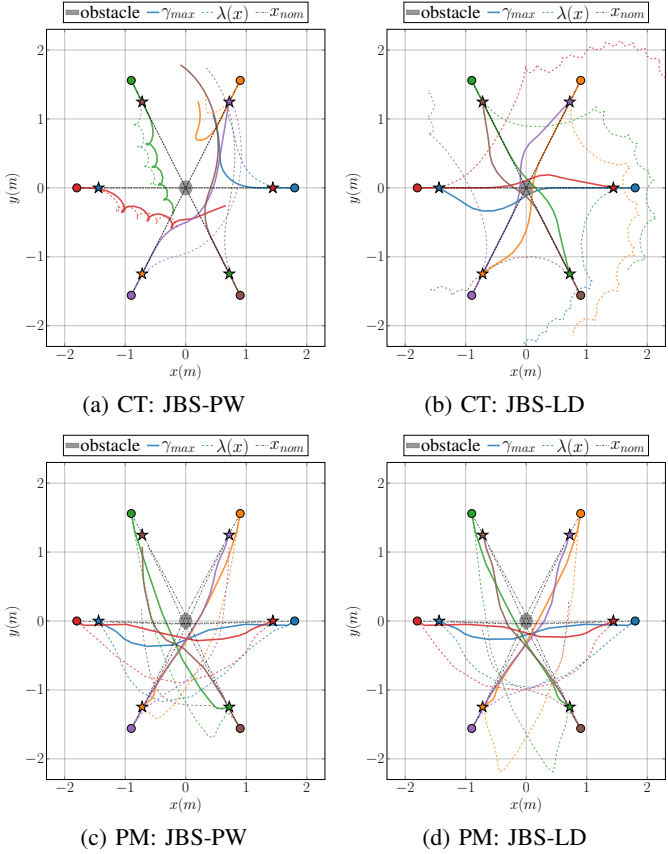

    \centering
    \begin{minipage}[t]{0.24\textwidth}
        \centering
        \resizebox{\linewidth}{!}{%
            \TrajectoryPlot{\CTPWBBTrajectoryData}%
        }
        \small (a) \gls{CT}: \gls{pwbb}
    \end{minipage}
    \hfill
    \begin{minipage}[t]{0.24\textwidth}
        \centering
        \resizebox{\linewidth}{!}{%
            \TrajectoryPlot{\CTLDBBTrajectoryData}%
        }
        \small (b) \gls{CT}: \gls{ldbb}
    \end{minipage}

    \vspace{2mm}

    \begin{minipage}[t]{0.24\textwidth}
        \centering
        \resizebox{\linewidth}{!}{%
            \TrajectoryPlot{\PMPWBBTrajectoryData}%
        }
        \small (c) \gls{PM}: \gls{pwbb}
    \end{minipage}
    \hfill
    \begin{minipage}[t]{0.24\textwidth}
        \centering
        \resizebox{\linewidth}{!}{%
            \TrajectoryPlot{\PMLDBBTrajectoryData}%
        }
        \small (d) \gls{PM}: \gls{ldbb}
    \end{minipage}

    \caption{Representative matched closed-loop trajectories for
    \gls{pwbb} and \gls{ldbb}. Filled circles indicate initial
    states and stars indicate goals. Solid colored curves denote the
    fixed-$\gamma_{\max}$ safety filter, dashed colored curves
    denote direct use of the synthesized state-dependent
    $\lambda(x)$, and black dash-dotted curves denote the nominal
    trajectories. The gray region represents the obstacle.}
    \label{fig:bb_rollouts}
\end{figure}

Under the fixed-$\gamma_{\max}$ filter, the \gls{qp} remains feasible and both methods remain
safe.
% in all matched trajectories.
For~\gls{CT}, \gls{pwbb} reaches the goal in $92$ of $500$
trajectories, whereas \gls{ldbb} reaches it in all $500$. For
\gls{PM}, \gls{pwbb} and \gls{ldbb} reach the goal in $366$
and $386$ of the $443$ matched trajectories, respectively.

\tabref{tab:bb_closed_loop_metrics} reports the median
$[Q_1,Q_3]$ of the closed-loop metrics. 
\begin{table}[t]
    \centering
\caption{Closed-loop performance under the
fixed-$\gamma_{\max}$ safety filter. Values are medians
$[Q_1,Q_3]$. $T_{\mathrm{reach}}$ is evaluated only over
successful trajectories for each method.}
    \label{tab:bb_closed_loop_metrics}
    \scriptsize
    \setlength{\tabcolsep}{2.5pt}

    \textbf{(a) \gls{CT}}

    \vspace{1mm}
    \begin{tabular}{@{}lcc@{}}
        \toprule
        Metric & \gls{pwbb} & \gls{ldbb} \\
        \midrule
        $J_{\mathrm{int},2}$
            & $19.62\,[17.05,26.20]$
            & $4.323\,[3.728,5.152]$ \\
        $J_g$
            & $58.61\,[44.17,64.47]$
            & $25.24\,[25.04,25.29]$ \\
        $T_{\mathrm{reach}}$ (s)
            & $17.93\,[16.70,24.35]$
            & $15.60\,[15.40,15.80]$ \\
        $d_{\min}$ (m)
            & $0.3117\,[0.2570,0.4127]$
            & $0.001074\,[0.001033,0.001211]$ \\
        \bottomrule
    \end{tabular}

    \vspace{2mm}

    \textbf{(b) \gls{PM}}
    
    \vspace{1mm}
    \begin{tabular}{@{}lcc@{}}
        \toprule
        Metric & \gls{pwbb} & \gls{ldbb} \\
        \midrule
        $J_{\mathrm{int},2}$
            & $0.7702\,[0.5228,1.360]$
            & $0.5947\,[0.3471,0.8677]$ \\
        $J_g$
            & $8.306\,[7.293,8.952]$
            & $8.317\,[7.333,8.526]$ \\
        $T_{\mathrm{reach}}$ (s)
            & $8.150\,[7.613,8.400]$
            & $8.150\,[7.750,8.350]$ \\
        $d_{\min}$ (m)
            & $0.1689\,[0.1002,0.2013]$
            & $0.09145\,[0.07877,0.1145]$ \\
        \bottomrule
    \end{tabular}
    \vspace{1mm}
    % \textbf{(b) \gls{PM}}

    % \vspace{1mm}
    % \begin{tabular}{@{}lcc@{}}
    %     \toprule
    %     Metric & \gls{pwbb} & \gls{ldbb} \\
    %     \midrule
    %     $J_{\mathrm{int},2}$
    %         & $0.7704\,[0.5228,1.360]$
    %         & $0.6416\,[0.4090,0.9179]$ \\
    %     $J_g$
    %         & $8.306\,[7.293,8.953]$
    %         & $8.337\,[7.335,8.536]$ \\
    %     $T_{\mathrm{reach}}$ (s)
    %         & $8.150\,[7.613,8.400]$
    %         & $8.200\,[7.700,8.350]$ \\
    %     $d_{\min}$ (m)
    %         & $0.1689\,[0.1002,0.2013]$
    %         & $0.1010\,[0.08246,0.1062]$ \\
    %     \bottomrule
    % \end{tabular}
    % \vspace{1mm}

    % \parbox{0.96\columnwidth}{%
    %     \scriptsize
    %     The common-success subsets used for
    %     $T_{\mathrm{reach}}$ contain $92$ trajectory pairs for
    %     \gls{CT} and $361$ for \gls{PM}.
    % }
\end{table}
The metrics
$J_{\mathrm{int},2}$, $J_g$, and $d_{\min}$ are evaluated
over all matched trajectories, whereas $T_{\mathrm{reach}}$ is
evaluated over successful trajectories for each method. \gls{ldbb} yields lower $J_{\mathrm{int},2}$ than
\gls{pwbb} on both systems. For \gls{CT}, it also yields lower
$J_g$ and $T_{\mathrm{reach}}$, whereas for \gls{PM} these
two metrics remain comparable between the methods. In both systems, \gls{ldbb} yields a lower median
$d_{\min}$ together with a narrower interquartile range,
indicating closer and more consistent obstacle clearance
across the evaluated approach directions.

Across all matched fixed-$\gamma_{\max}$ trajectories, the
minimum barrier value evaluated on the dense RK4 integration
grid remains positive for every method-system pair.
\figref{fig:bb_barrier_time} shows the evaluated barrier
value for six trajectories, and throughout these trajectories $B(x(t))$ remains nonnegative.
% \input{figures/experiment/barrier_vs_t}
% ============================================================
% barrier_vs_t.tex -- four-panel barrier-time assembler
% ============================================================

\definecolor{barone}{rgb}{0.1216,0.4667,0.7059}
\definecolor{bartwo}{rgb}{1.0000,0.4980,0.0549}
\definecolor{barthree}{rgb}{0.1725,0.6275,0.1725}
\definecolor{barfour}{rgb}{0.8392,0.1529,0.1569}
\definecolor{barfive}{rgb}{0.5804,0.4039,0.7412}
\definecolor{barsix}{rgb}{0.5490,0.3373,0.2941}

\pgfplotsset{
barrier axis/.style={
point meta max={nan},
point meta min={nan},
filter discard warning=false,
legend cell align={left},
legend columns={3},
legend style={color=black,draw opacity=1,line width=1,solid,fill=white,fill opacity=1,text opacity=1,font={\fontsize{25pt}{23.4pt}\selectfont},text=black,cells={anchor=center},at={(0.5,1.02)},anchor=south},
axis background/.style={fill=white,opacity=1},
anchor=north west,
xshift=1mm,
yshift=5mm,
width=150.4mm,
height=105.6mm,
scaled x ticks=false,
xlabel={time (s)},
xlabel style={at={(ticklabel cs:0.5)},anchor=near ticklabel,font={\fontsize{20pt}{15.6pt}\selectfont},color=black},
xmajorgrids=true,
xtick align=inside,
xticklabel style={font={\fontsize{20pt}{23.4pt}\selectfont},color=black},
x grid style={color=black,draw opacity=0.25,line width=0.5,solid},
x axis line style={color=black,draw opacity=1,line width=1,solid},
scaled y ticks=false,
ylabel={$B(x)$},
ylabel style={at={(ticklabel cs:0.5)},anchor=near ticklabel,font={\fontsize{20pt}{15.6pt}\selectfont},color=black},
ymajorgrids=true,
ytick align=inside,
yticklabel style={font={\fontsize{20pt}{23.4pt}\selectfont},color=black},
y grid style={color=black,draw opacity=0.25,line width=0.5,solid},
y axis line style={color=black,draw opacity=1,line width=1,solid},
colorbar=false
},
gamma barrier/.style={draw opacity=0.85,line width=2.2,solid},
lambda barrier/.style={draw opacity=0.75,line width=1.1,dashed},
zero barrier/.style={color=black,draw opacity=1,line width=1,dashed}
}

% #1 csv, #2 series_id, #3 rollout_idx, #4 color, #5 style
\newcommand{\AddFilteredBarrier}[5]{%
\addplot[color=#4,unbounded coords=jump,#5] table[col sep=comma,x expr={((\thisrow{series_id}==#2)&&(\thisrow{rollout_idx}==#3)) ? \thisrow{t} : nan},y expr={((\thisrow{series_id}==#2)&&(\thisrow{rollout_idx}==#3)) ? \thisrow{B} : nan}]{#1};%
}
% \newcommand{\AddFilteredBarrier}[5]{%
% \addplot[
%     color=#4,
%     unbounded coords=jump,
%     restrict expr to domain={\thisrow{series_id}}{#2:#2},
%     restrict expr to domain={\thisrow{rollout_idx}}{#3:#3},
%     #5
% ]
% table[
%     col sep=comma,
%     x=t,
%     y=B
% ]{#1};%
% }

% Optional argument is used only for panel-specific axis limits/ticks.
\newcommand{\BarrierPlot}[2][]{%
\begin{tikzpicture}[/tikz/background rectangle/.style={fill=white},show background rectangle]%
\begin{axis}[barrier axis,#1]%
% gamma_max: original thick solid style
\AddFilteredBarrier{#2}{0}{1}{barone}{gamma barrier}%
% \addlegendentry{$\gamma_{max}$ \quad}%
\AddFilteredBarrier{#2}{0}{2}{bartwo}{gamma barrier,forget plot}%
\AddFilteredBarrier{#2}{0}{3}{barthree}{gamma barrier,forget plot}%
\AddFilteredBarrier{#2}{0}{4}{barfour}{gamma barrier,forget plot}%
\AddFilteredBarrier{#2}{0}{5}{barfive}{gamma barrier,forget plot}%
\AddFilteredBarrier{#2}{0}{6}{barsix}{gamma barrier,forget plot}%
% lambda(x): original thin dashed style
% \AddFilteredBarrier{#2}{1}{1}{barone}{lambda barrier}%
% \addlegendentry{$\lambda(x)$ \quad}%
% \AddFilteredBarrier{#2}{1}{2}{bartwo}{lambda barrier,forget plot}%
% \AddFilteredBarrier{#2}{1}{3}{barthree}{lambda barrier,forget plot}%
% \AddFilteredBarrier{#2}{1}{4}{barfour}{lambda barrier,forget plot}%
% \AddFilteredBarrier{#2}{1}{5}{barfive}{lambda barrier,forget plot}%
% \AddFilteredBarrier{#2}{1}{6}{barsix}{lambda barrier,forget plot}%
% B=0 reference line
\addplot[zero barrier] coordinates {(0,0) (30,0)};%
% \addlegendentry{$B=0$}%
\end{axis}%
\end{tikzpicture}%
}

% 20260823_043800_yaml_reproduction
% 20260823_093703_yaml_reproduction

% Change only these four paths if your CSV names differ.
\newcommand{\CTPWBarrierData}{figures/experiment/coordturn/csv_data/pwbb_bt.csv}
\newcommand{\CTLDBarrierData}{figures/experiment/coordturn/csv_data/ldbb_bt.csv}
\newcommand{\PMPWBarrierData}{figures/experiment/planar_multirotor/csv_data/pwbb_bt.csv}
\newcommand{\PMLDBarrierData}{figures/experiment/planar_multirotor/csv_data/ldbb_bt.csv}

\begin{figure}[t]
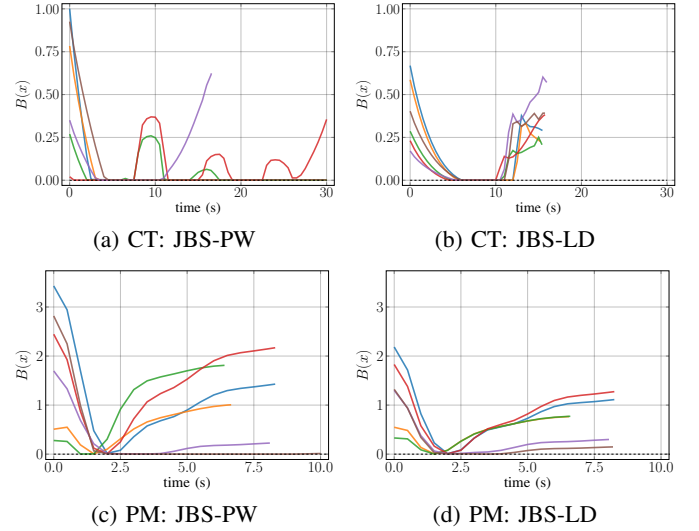

    \centering

    \begin{minipage}[t]{0.24\textwidth}
        \centering
        \resizebox{\linewidth}{!}{%
            \BarrierPlot[
                xmin=-0.9,xmax=30.9,
                xtick={0,10,20,30},
                ymin=-0.03,ymax=1.03,
                ytick={0,0.25,0.5,0.75,1},
                yticklabels={{$0.00$,$0.25$,$0.50$,$0.75$,$1.00$}}
            ]{\CTPWBarrierData}%
        }
        \small (a) \gls{CT}: \gls{pwbb}
    \end{minipage}
    \hfill
    \begin{minipage}[t]{0.24\textwidth}
        \centering
        \resizebox{\linewidth}{!}{%
            \BarrierPlot[
                xmin=-0.9,xmax=30.9,
                xtick={0,10,20,30},
                ymin=-0.03,ymax=1.03,
                ytick={0,0.25,0.5,0.75,1},
                yticklabels={{$0.00$,$0.25$,$0.50$,$0.75$,$1.00$}}
            ]{\CTLDBarrierData}%
        }
        \small (b) \gls{CT}: \gls{ldbb}
    \end{minipage}
    \vspace{2mm}

    \vfill

    \begin{minipage}[t]{0.24\textwidth}
        \centering
        \resizebox{\linewidth}{!}{%
            \BarrierPlot[
                xmin=-0.3,xmax=10.3,
                xtick={0,2.5,5,7.5,10},
                xticklabels={{$0.0$,$2.5$,$5.0$,$7.5$,$10.0$}},
                ymin=-0.10887833567921246,
                ymax=3.7381561916529584,
                ytick={0,1,2,3}
            ]{\PMPWBarrierData}%
        }
        \small (c) \gls{PM}: \gls{pwbb}
    \end{minipage}
    \hfill
    \begin{minipage}[t]{0.24\textwidth}
        \centering
        \resizebox{\linewidth}{!}{%
            \BarrierPlot[
                xmin=-0.3,xmax=10.3,
                xtick={0,2.5,5,7.5,10},
                xticklabels={{$0.0$,$2.5$,$5.0$,$7.5$,$10.0$}},
                ymin=-0.10887833567921246,
                ymax=3.7381561916529584,
                ytick={0,1,2,3}
            ]{\PMLDBarrierData}%
        }
        \small (d) \gls{PM}: \gls{ldbb}
    \end{minipage}

    % \caption{Solid curves show $B(x(t))$ along the densely integrated
    % trajectories under the fixed-$\gamma_{\max}$ safety filter, while
    % dashed colored curves show the corresponding trajectories obtained
    % using the synthesized state-dependent $\lambda(x)$. The horizontal
    % black dashed line denotes $B=0$.}
    \caption{Solid curves show $B(x(t))$ along the densely integrated
    trajectories under the fixed-$\gamma_{\max}$ safety filter; the horizontal
    black dashed line denotes $B=0$.}
    \label{fig:bb_barrier_time}
\end{figure}

Finally, the combined online computation of the local
inter-sample margin~\eqref{eq:zoh_nu} and the \gls{cbf}-\gls{qp} remains small relative to
the $T_s=50\,\mathrm{ms}$ sampling period. After warm-up,
the median computation time ranges from $1.86$ to
$9.16\,\mathrm{ms}$ across the two systems and methods, while
the corresponding $95$th percentiles range from $4.54$ to
$13.36\,\mathrm{ms}$.

\subsection{Successive-Barrier Synthesis Results}
To isolate the effect of model representation, we compare \gls{ldsb} and the \gls{pwsb} baseline across physical synthesis sample sizes 
$n\in\{50,200,500,1000\}$, with two synthesis
levels used in all these experiments. 
Certified coverage is evaluated as defined in
~\appref{app:evaluation_metrics}.
We additionally report
the number of barriers, total synthesis time, and median memory
usage. ~\figref{fig:sb_metric_panels} summarizes these
results.

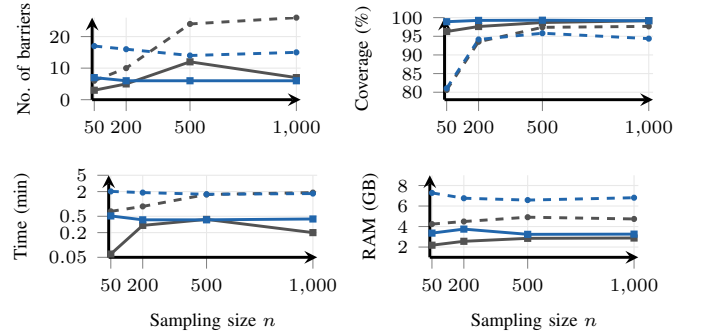
\begin{figure}[!t]
    \centering
    \definecolor{sweepgray}{gray}{0.32}
\definecolor{sweepblue}{HTML}{2166AC}
\pgfplotsset{
    sweep axis/.style={
        width=\linewidth,
        height=1.05in,
        xmin=40, xmax=1020,
        xtick={50,200,500,1000},
        grid=both,
        major grid style={draw=gray!18},
        minor grid style={draw=gray!8},
        axis lines=left,
        tick align=outside,
        scaled ticks=false,
        tick label style={font=\scriptsize},
        label style={font=\scriptsize},
        title style={font=\scriptsize},
        legend style={font=\tiny, draw=none, fill=none},
        line width=1.1pt,
    },
    pwsb line/.style={
        draw=sweepgray,
        dashed,
        mark=*,
        mark size=0.6pt,
        mark options={solid,fill=sweepgray},
    },
    ldsb line/.style={
        draw=sweepgray,
        solid,
        mark=square*,
        mark size=0.8pt,
        mark options={solid,fill=sweepgray},
    },
    pwsb line pm/.style={
        draw=sweepblue,
        dashed,
        mark=*,
        mark size=0.6pt,
        mark options={solid,fill=sweepblue},
    },
    ldsb line pm/.style={
        draw=sweepblue,
        solid,
        mark=square*,
        mark size=0.8pt,
        mark options={solid,fill=sweepblue},
    },
}

% Data source: experiment/externel_comparsion/results/sweep_table.csv
% ("... via NBInclude" mosek rows; Planar Multirotor uses the v/omega<=2.0/1.0
% capped-bound rows, matching pm_physical_bounds — the uncapped rows are
% degenerate (0% coverage) and are not used here).
\begin{minipage}[t]{0.49\linewidth}
    \centering
    \begin{tikzpicture}
        \begin{axis}[
            sweep axis,
            ylabel={No.\ of barriers},
            ymin=0,
            legend style={at={(0.03,0.97)}, anchor=north west, /tikz/every even column/.append style={column sep=0.3em}},
            legend columns=1,
        ]
            \addplot[pwsb line] coordinates {
                (50,6) (200,10) (500,24) (1000,26)
            };
            % \addlegendentry{PWSB (CoordTurn)}
            \addplot[ldsb line] coordinates {
                (50,3) (200,5) (500,12) (1000,7)
            };
            % \addlegendentry{LDSB (CoordTurn)}
            \addplot[pwsb line pm] coordinates {
                (50,17) (200,16) (500,14) (1000,15)
            };
            % \addlegendentry{PWSB (Multirotor)}
            \addplot[ldsb line pm] coordinates {
                (50,7) (200,6) (500,6) (1000,6)
            };
            % \addlegendentry{LDSB (Multirotor)}
        \end{axis}
    \end{tikzpicture}
\end{minipage}\hfill
\begin{minipage}[t]{0.49\linewidth}
    \centering
    \begin{tikzpicture}
        \begin{axis}[
            sweep axis,
            ylabel={Coverage (\%)},
            ymin=78, ymax=100,
        ]
            \addplot[pwsb line] coordinates {
                (50,80.61) (200,93.53) (500,97.40) (1000,97.66)
            };
            \addplot[ldsb line] coordinates {
                (50,96.28) (200,97.62) (500,98.69) (1000,99.18)
            };
            \addplot[pwsb line pm] coordinates {
                (50,81.05) (200,94.15) (500,95.82) (1000,94.37)
            };
            \addplot[ldsb line pm] coordinates {
                (50,98.91) (200,99.26) (500,99.31) (1000,99.16)
            };
        \end{axis}
    \end{tikzpicture}
\end{minipage}

\vspace{0.55em}

\begin{minipage}[t]{0.49\linewidth}
    \centering
    \begin{tikzpicture}
        % \begin{axis}[
        %     sweep axis,
        %     % title={Time},
        %     ylabel={Time (min)},
        %     xlabel={Sampling size $n$},
        %     ymode=log,
        %     log basis y=10,
        %     log ticks with fixed point,
        %     ymin=0.5,
        %     ymax=10,
        %     ytick={0.5,1,2,3},
        % ]
        \begin{axis}[
    sweep axis,
    ylabel={Time (min)},
    xlabel={Sampling size $n$},
    ymode=log,
    log basis y=10,
    log ticks with fixed point,
    ymin=0.05,
    ymax=5,
    ytick={0.05,0.2,0.5,2,5},
]
            \addplot[pwsb line] coordinates {
                (50,0.66) (200,0.87) (500,1.70) (1000,1.89)
            };
            \addplot[ldsb line] coordinates {
                (50,0.06) (200,0.30) (500,0.42) (1000,0.20)
            };
            \addplot[pwsb line pm] coordinates {
                (50,2.02) (200,1.89) (500,1.72) (1000,1.79)
            };
            \addplot[ldsb line pm] coordinates {
                (50,0.51) (200,0.41) (500,0.41) (1000,0.43)
            };
        \end{axis}
    \end{tikzpicture}
\end{minipage}\hfill
\begin{minipage}[t]{0.49\linewidth}
    \centering
    \begin{tikzpicture}
        \begin{axis}[
            sweep axis,
            ylabel={RAM (GB)},
            xlabel={Sampling size $n$},
            ymin=1.0, ymax=9.0,
        ]
            \addplot[pwsb line] coordinates {
                (50,4.24) (200,4.49) (500,4.91) (1000,4.74)
            };
            \addplot[ldsb line] coordinates {
                (50,2.18) (200,2.56) (500,2.85) (1000,2.88)
            };
            \addplot[pwsb line pm] coordinates {
                (50,7.28) (200,6.76) (500,6.58) (1000,6.81)
            };
            \addplot[ldsb line pm] coordinates {
                (50,3.36) (200,3.75) (500,3.24) (1000,3.26)
            };
        \end{axis}
    \end{tikzpicture}
\end{minipage}
    \caption{\gls{sbs} versus the number of
    physical synthesis samples $n$. The panels report 
    barrier-bank size, coverage, synthesis time on a logarithmic scale, and median memory usage for
    \gls{CT} (gray) and the \gls{PM} (blue).
    \gls{pwsb} is dashed and \gls{ldsb} is solid.}
    \label{fig:sb_metric_panels}
\end{figure}

Across all tested sample sizes, \gls{ldsb} achieves higher estimated certified coverage with fewer barriers than \gls{pwsb}, while
also requiring less solver time and median memory. The largest
coverage differences occur at $n=50$, where coverage
% increases by $15.67\%$ points for \gls{CT} and $17.86\%$ points for the \gls{PM}.
increases from $80.61\%$ to $96.28\%$ for \gls{CT} and
from $81.05\%$ to $98.91\%$ for the \gls{PM}.
At $n=1000$, \gls{ldsb} certifies more than $99\%$ of the
evaluation samples for both models, compared with $97.66\%$
for \gls{CT} and $94.37\%$ for \gls{PM} under \gls{pwsb}.
It uses $7$ and $6$ barriers for \gls{CT} and \gls{PM},
respectively, compared with $26$ and $15$ for \gls{pwsb}.
At the same sample size, the \gls{ldsb}/\gls{pwsb} solver
times are $0.20/1.89\,\mathrm{min}$ for \gls{CT} and
$0.43/1.79\,\mathrm{min}$ for \gls{PM}.

To complement these aggregate results,
~\figref{fig:sb_slices} shows representative slices of the unions of successive-barrier
certified sets, with a visibly larger certified union for \gls{ldsb}
on both benchmarks.
A state is displayed as certified whenever at least one barrier
in the corresponding bank is nonnegative.

\begin{figure}[t]
    \centering
    \begin{minipage}[t]{0.24\textwidth}
        \centering
        \includegraphics[width=\linewidth]{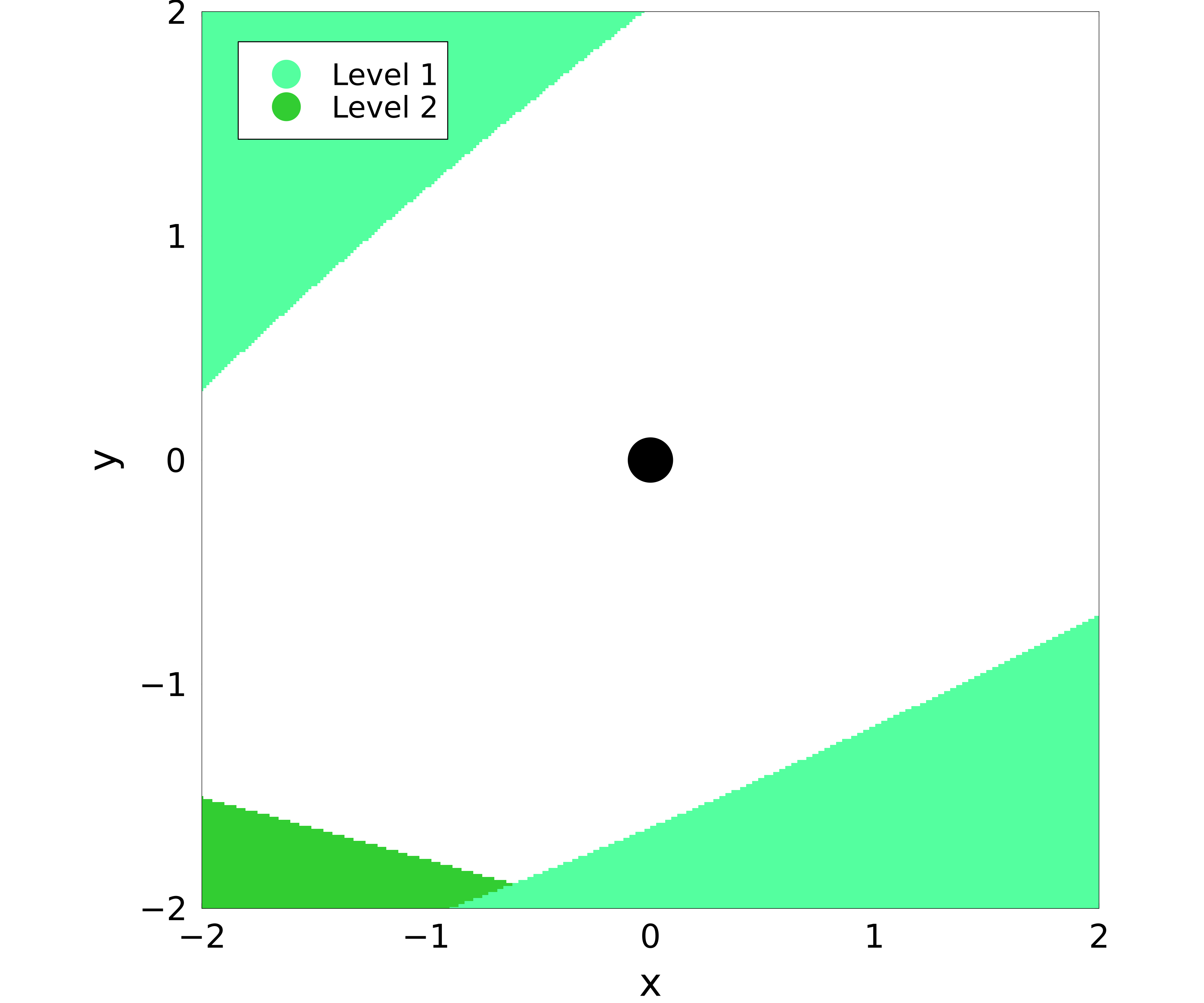}
        \small (a) \gls{CT}: \gls{pwsb}
    \end{minipage}
    \hfill
    \begin{minipage}[t]{0.24\textwidth}
        \centering
        \includegraphics[width=\linewidth]{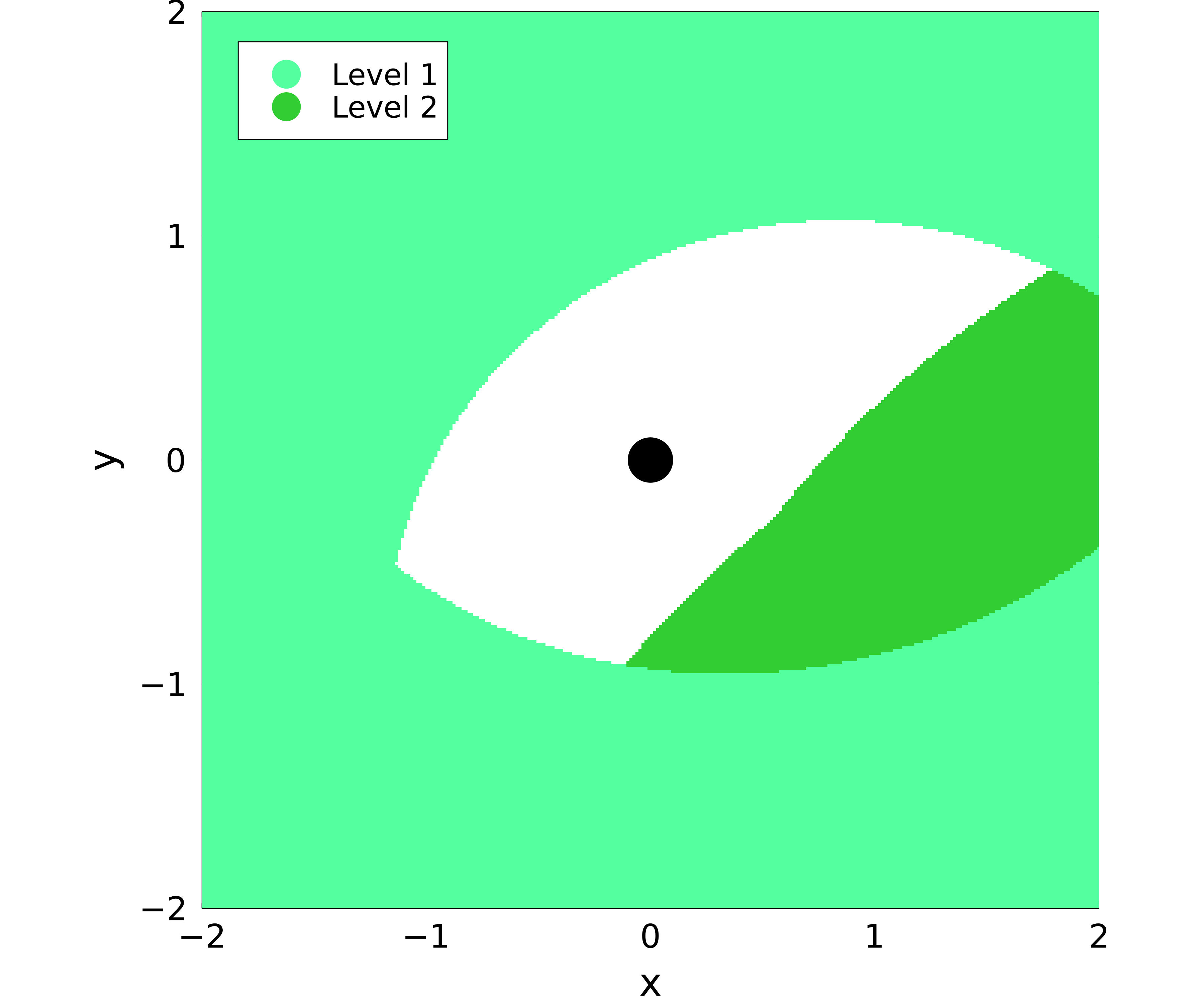}
        \small (b) \gls{CT}: \gls{ldsb}
    \end{minipage}
    \vfill
    \begin{minipage}[t]{0.24\textwidth}
        \centering
        \includegraphics[width=\linewidth]{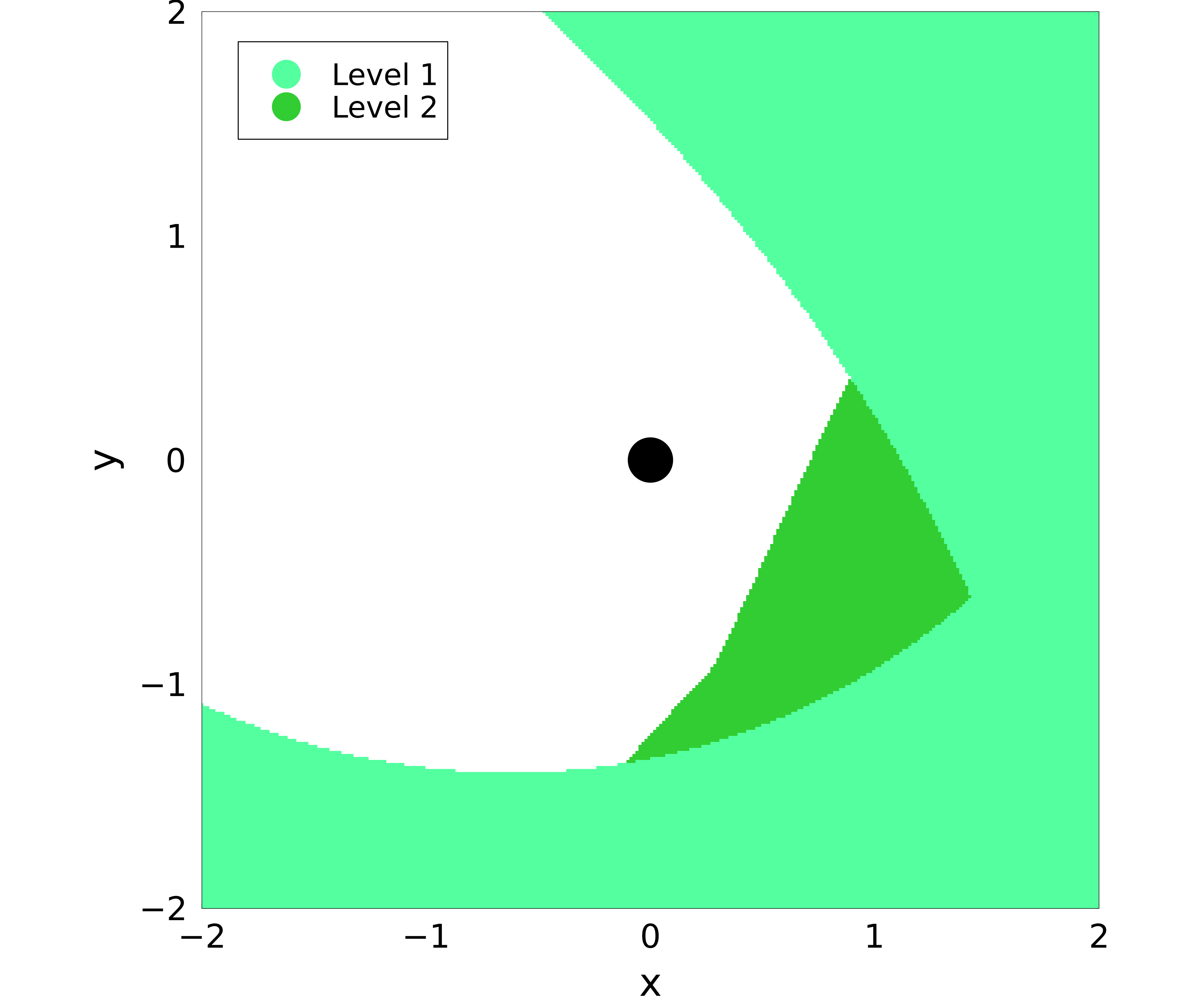}
        \small (c) \gls{PM}: \gls{pwsb}
    \end{minipage}
    \hfill
    \begin{minipage}[t]{0.24\textwidth}
        \centering
        \includegraphics[width=\linewidth]{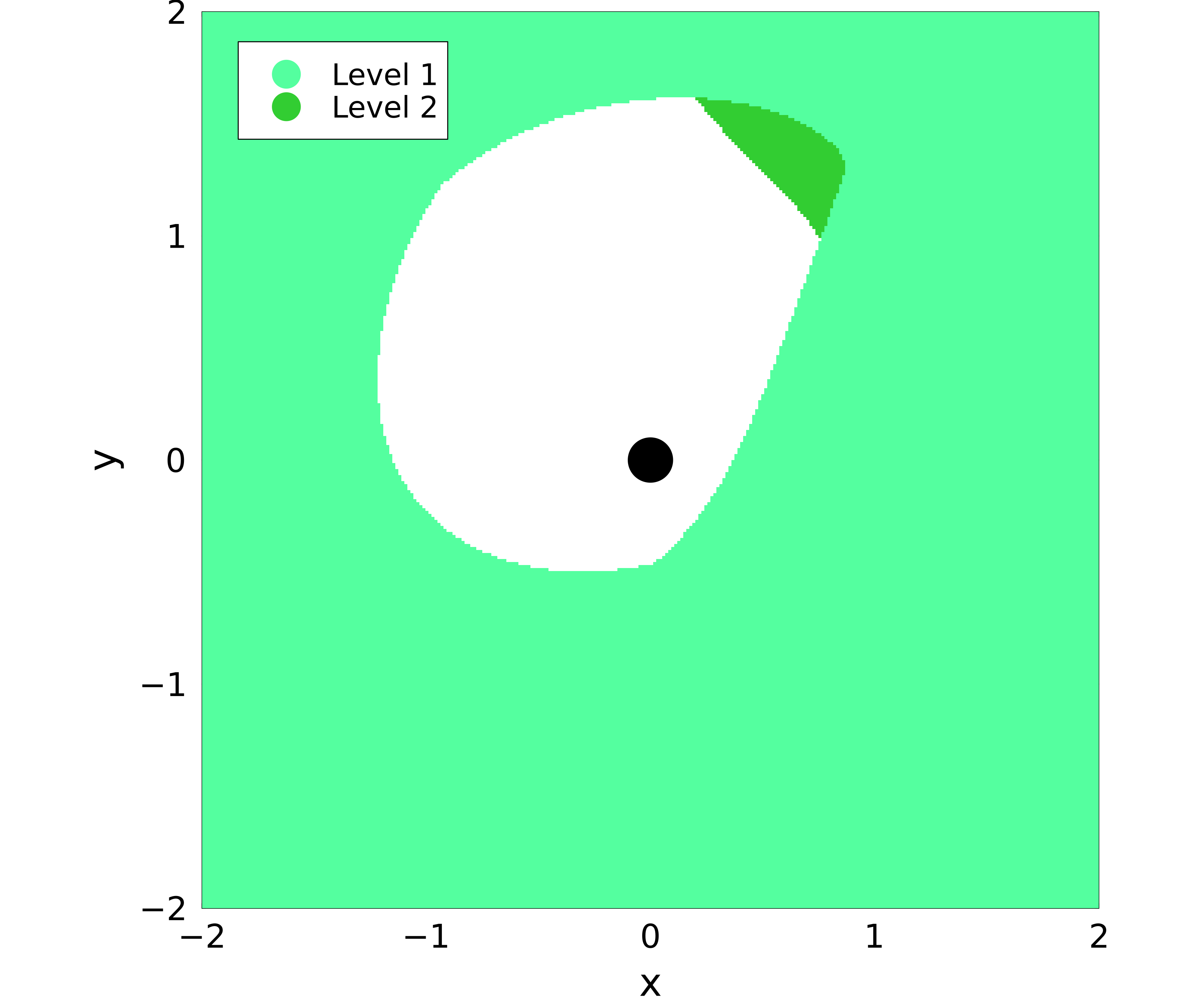}
        \small (d) \gls{PM}: \gls{ldsb}
    \end{minipage}

    \caption{Representative slices of the union of the barrier-certified sets synthesized with $n=1000$. The black disk denotes
    the obstacle, and the displayed certified region contains
    states for which at least one barrier in the bank satisfies
    $B_{\ell,j}(x)\geq0$. }
    \label{fig:sb_slices}
\end{figure}
% \FloatBarrier
\section{Discussion}
\label{sec:discussion}

When considering all these results together, the main finding is that the choice of dynamics representation affects both certified coverage and synthesis cost, across both synthesis methods. In the piecewise formulation, reducing approximation error requires finer sectorization and additional \gls{sos} conditions, whereas the lifted formulation represents the nonlinear dynamics directly through lifted variables and algebraic constraints, preserving the unit-circle geometry. Thus, synthesis cost depends on how the dynamics are encoded, not only on the state dimension. More specifically, for \gls{bbs}, the coverage versus time~\figref{fig:bb_synthesis_convergence} curves show that
\gls{ldbb} accumulates certified coverage more rapidly than
\gls{pwbb} for $Q = 16$. In \gls{sbs}, the more direct representation allows each accepted barrier to certify a larger portion of the sampled domain, so fewer barriers and fewer \gls{sos} solves are required to achieve higher coverage than \gls{pwsb}.

The closed-loop results show that \gls{ldbb} reduces intervention and gives more consistent spatial behavior across approach directions, with the effect most pronounced for \gls{CT}. In \gls{CT}, the sector-wise approximation and error constraints in \gls{pwbb} introduce heading-dependent conservativeness into the synthesized barrier, causing the safety filter to initiate obstacle avoidance at different spatial locations for different approach directions. In contrast, \gls{ldbb} avoids heading sectorization and produces more consistent avoidance trajectories and obstacle clearance. For \gls{PM}, the two representations give similar goal-tracking behavior, while \gls{ldbb} still requires less intervention and yields closer and more consistent clearance. 
% This smaller difference in task-space motion is consistent with the cascaded controller, where attitude acts as an intermediate variable between position tracking and control input.
This smaller task-space difference for \gls{PM} may partly reflect its
second-order translational dynamics and cascaded position--attitude
control, through which differences in safety-filter intervention affect
position only after acting through thrust, attitude, and velocity.

A limitation of the sampled-data filter is the trade-off between tight inter-sample margins and recursive-feasibility guarantees. 
In the closed-loop simulations, the inter-sample margin \(\nu_k\) is
computed from a local numerical estimate of \(\eta_k\), rather than
from a global bound over \(\mathcal C\), reducing conservatism.
% The rollouts use a local numerical estimate of \(\eta_k\) to reduce \gls{qp} tightening and avoid the conservativeness of a global bound.
% The rollouts use a local numerical estimate of \(\eta_k\) to compute
% the inter-sample margin \(\nu_k\), making it less conservative than a
% margin based on a global bound over \(\mathcal C\).
% A recursive-feasibility guarantee would instead require a certified global margin bound together with a sufficient witness margin, generally at the cost of increased conservatism, as discussed in Remark~\ref{rem:zoh_recursive}.
A whole-set recursive-feasibility guarantee would additionally
require a certified lower bound on the witness margin that dominates
the required inter-sample margin throughout the certified set, as
discussed in Remark~\ref{rem:zoh_recursive}.
\section{Conclusions}
\label{sec:conclusion}
Safety filtering for robotic systems requires both a valid barrier certificate and a way to enforce it under \gls{zoh} control. On the synthesis side, non-polynomial robotic
dynamics are not directly compatible with \gls{sos}-based methods.
We address this through exact polynomial lifting. For the
trigonometric systems considered here, lifting replaces sector-wise
polynomial approximations and explicit error bounds with polynomial
lifted dynamics and lifting-induced algebraic constraints. This
removes sector resolution as a design parameter and avoids the
representation trade-off between approximation fidelity and the
number of local certification conditions. Especially for rotational coordinates, the lifted unit-circle representation makes periodic
geometry intrinsic, rather than requiring explicit endpoint-consistency
constraints.

On the deployment side, we establish a fixed runtime barrier-rate bridge from boundary-based synthesis to~\gls{zoh} implementation, ensuring compatibility with the synthesized certificate while accounting for inter-sample evolution. Together, these results extend \gls{sos}-based barrier certification to
robotic systems whose non-polynomial dynamics admit an exact polynomial
lifting, and make the resulting certificate usable as a runtime safety
filter whenever the \gls{qp} remains feasible at each sampling instant.

Future work should incorporate the desired sampling period and corresponding safety margin directly into synthesis, and investigate whether the structural advantages of lifting 
persist for configuration spaces beyond $\mathrm{SO}(2)$, such as $\mathrm{SO}(3)$ or $\mathrm{SE}(3)$.

\appendix

\makeatletter
\renewcommand{\p@subsection}{}
\renewcommand{\p@subsubsection}{\thesubsection-}
\makeatother

% \section{Appendix}
% \label{sec:appendix}

\subsection{Experimental Models}
\label{app:experimental_models}

\subsubsection{Coordinated Turn Model}
\label{app:coordturn_details}

% The piecewise-polynomial and exact lifted representations used in
% the coordinated turn experiments are constructed from the original
% dynamics in~\eqref{eq:ct_original_dynamics}.
Both representations are constructed from the original dynamics in~\eqref{eq:ct_original_dynamics}.

\paragraph{Piecewise-polynomial representation}
For an integer $Q$ divisible by four, the heading domain is
partitioned into $Q$ sectors of equal width,
\begin{equation}
\label{eq:ct_sector_partition}
    \mathcal I_q
    :=
    [\ell_q,r_q],
    \ 
    \ell_q:=-\pi+(q-1)\frac{2\pi}{Q},
    \ 
    r_q:=\ell_q+\frac{2\pi}{Q},
\end{equation}
for $q=1,\ldots,Q$. On each sector, $\cos\theta$ and
$\sin\theta$ are replaced by the sector-wise affine
approximations that interpolate them at the sector endpoints,
$$
    c_q(\theta)
    :=
    \cos\ell_q
    +
    \frac{\cos r_q-\cos\ell_q}{2\pi/Q}
    (\theta-\ell_q),
$$
$$
    s_q(\theta)
    :=
    \sin\ell_q
    +
    \frac{\sin r_q-\sin\ell_q}{2\pi/Q}
    (\theta-\ell_q),
$$
each used only for $\theta\in\mathcal I_q$. The associated
approximation residuals are
$$
    e_{c,q}(\theta):=\cos\theta-c_q(\theta),
    \qquad
    e_{s,q}(\theta):=\sin\theta-s_q(\theta),
$$
and are bounded component-wise on each sector by
\begin{equation}
\label{eq:ct_trig_error_bound}
    \varepsilon_{c,q}
    :=
    \max_{\theta\in\mathcal I_q}
    \bigl|e_{c,q}(\theta)\bigr|,
    \qquad
    \varepsilon_{s,q}
    :=
    \max_{\theta\in\mathcal I_q}
    \bigl|e_{s,q}(\theta)\bigr| .
\end{equation}
Requiring $Q$ to be divisible by four places every inflection
point of $\cos$ and $\sin$ at a sector endpoint. Consequently
no sector contains an inflection point, each residual has constant
curvature sign on $\mathcal I_q$ and vanishes at both endpoints,
and $|e_{c,q}|$ and $|e_{s,q}|$ are therefore unimodal with a
single interior maximum. The bounds
in~\eqref{eq:ct_trig_error_bound} are computed by golden-section
search on each sector and rounded outward.\footnote{%
For $Q=4$ the bound admits a closed form. On $[0,\pi/2]$,
with $\kappa:=2/\pi$, the residual
$e(\theta)=\cos\theta-1+\kappa\theta$ is stationary where
$\sin\theta^\star=\kappa$, giving
$\varepsilon=\sqrt{1-\kappa^2}-\kappa\arccos\kappa
\approx0.2105137$.
Symmetry makes this value common to both components in every
sector.}
Table~\ref{tab:ct_sector_errors} reports the values used in the
experiments.
\begin{table}[t]
    \centering
    \caption{Sector-wise trigonometric approximation-error bounds
    used in the experiments. Sectors listed in the same row share
    identical bounds by symmetry.}
    \label{tab:ct_sector_errors}
    \footnotesize
    \setlength{\tabcolsep}{5pt}
    \begin{tabular}{@{}cccc@{}}
        \toprule
        $Q$ & Sectors $q$
        & $\varepsilon_{c,q}$
        & $\varepsilon_{s,q}$ \\
        \midrule
        4
        & $1\!:\!4$
        & 0.210514 & 0.210514 \\
        \addlinespace
        \multirow{2}{*}{8}
        & $1,4,5,8$
        & 0.070378 & 0.029823 \\
        & $2,3,6,7$
        & 0.029823 & 0.070378 \\
        \addlinespace
        \multirow{4}{*}{16}
        & $1,8,9,16$
        & 0.018846 & 0.003845 \\
        & $2,7,10,15$
        & 0.015984 & 0.010701 \\
        & $3,6,11,14$
        & 0.010701 & 0.015984 \\
        & $4,5,12,13$
        & 0.003845 & 0.018846 \\
        \bottomrule
    \end{tabular}
\end{table}
Using~\eqref{eq:ct_trig_error_bound}, the trigonometric terms are
written on each sector as
\begin{equation}
\label{eq:ct_trig_residuals}
\begin{aligned}
    \cos\theta
        &= c_q(\theta)+\varepsilon_{c,q}\delta_c,\\
    \sin\theta
        &= s_q(\theta)+\varepsilon_{s,q}\delta_s,
\end{aligned}
\end{equation}
where $\delta=[\delta_c,\delta_s]^\top$ collects the
normalized residuals. Since $e_{c,q}$ and $e_{s,q}$ have
constant sign on $\mathcal I_q$, let $e_{\nu,q}$ denote the
corresponding residual for $\nu\in\{c,s\}$. Then
\begin{equation}
\label{eq:ct_error_set}
\delta_\nu\in
\begin{cases}
[0,1],  & e_{\nu,q}\geq0 \text{ on }\mathcal I_q,\\
[-1,0], & e_{\nu,q}\leq0 \text{ on }\mathcal I_q .
\end{cases}
\end{equation}
The Cartesian product of these two intervals is denoted by
$\Delta_q$.

Substitution into~\eqref{eq:ct_original_dynamics} gives the
piecewise-polynomial differential inclusion
\begin{equation}
\label{eq:ct_piecewise_model_appendix}
\dot x=
\begin{bmatrix}
v\!\left(c_q(\theta)+\varepsilon_{c,q}\delta_c\right)\\
v\!\left(s_q(\theta)+\varepsilon_{s,q}\delta_s\right)\\
u_1\\
\omega\\
u_2
\end{bmatrix},
\end{equation}
for $\theta\in\mathcal I_q$, and $\delta\in\Delta_q$.
For each fixed $u$,
\eqref{eq:ct_piecewise_model_appendix} is affine in $\delta$, so
the robust conditions reduce to the four vertices of
$\Delta_q$ in each sector.
\paragraph{Exact lifted representation}
Introduce the auxiliary variables
\begin{equation}
\label{eq:ct_auxiliary_variables}
    c:=\cos\theta,
    \qquad
    s:=\sin\theta .
\end{equation}
After this substitution, $\theta$ does not appear explicitly in
the dynamics, the obstacle specification, or the input
constraints. As in Remark~\ref{rem:reduced_orientation}, we therefore use a
reduced lifting map
\begin{equation}
\label{eq:ct_lifting_map}
    z=\phi(x)
    :=
    [p_x,p_y,v,c,s,\omega]^\top .
\end{equation}
The chain rule gives the auxiliary dynamics
\begin{equation}
\label{eq:ct_auxiliary_dynamics}
    \dot c=-s\omega,
    \qquad
    \dot s=c\omega ,
\end{equation}
so that the exact lifted model is
$$
\dot z=
\begin{bmatrix}
    vc &
    vs &
    u_1 &
    -s\omega &
    c\omega &
    u_2
\end{bmatrix}^\top ,
$$
subject to the lifting-induced algebraic equality
\begin{equation}
\label{eq:ct_lifting_equality}
    \hat h(z)
    :=
    c^2+s^2-1=0 .
\end{equation}
No additional lifting-induced inequality is required; the obstacle and input constraints are embedded unchanged since they depend only on position and control, respectively.

\subsubsection{Planar Multirotor Model}
\label{app:multirotor_details}

% The piecewise-polynomial and exact lifted representations used in
% the planar multirotor experiments are constructed from the
% original dynamics in~\eqref{eq:pm_original_dynamics}.
Both representations are constructed from the original dynamics in~\eqref{eq:pm_original_dynamics}.

\paragraph{Piecewise-polynomial representation}
We reuse the sector partition, the affine approximations
$c_q(\theta)$ and $s_q(\theta)$, the sector-wise bounds
$\varepsilon_{c,q}$ and $\varepsilon_{s,q}$, and the
set $\Delta_q$ introduced in
\eqref{eq:ct_sector_partition}--\eqref{eq:ct_error_set}, since
these depend only on $Q$ and not on the system. Writing the
total thrust as $T:=u_1+u_2$ and
substituting~\eqref{eq:ct_trig_residuals}
into~\eqref{eq:pm_original_dynamics} gives
% We reuse $\mathcal I_q$, $c_q$, $s_q$, $\varepsilon_{c,q}$,
% $\varepsilon_{s,q}$, and $\Delta_q$ from~\eqref{eq:ct_sector_partition}--%
% \eqref{eq:ct_error_set}, which depend only on $Q$.
% With $T := u_1 + u_2$, substituting~\eqref{eq:ct_trig_residuals}
% into~\eqref{eq:pm_original_dynamics} gives
\begin{equation*}
\label{eq:pm_piecewise_model_appendix}
\dot x=
\begin{bmatrix}
    v_x\\
    v_y\\
    T\!\left(s_q(\theta)+\varepsilon_{s,q}\delta_s\right)\\
    T\!\left(c_q(\theta)+\varepsilon_{c,q}\delta_c\right)-g_0\\
    \omega\\
    u_1-u_2
\end{bmatrix},
\end{equation*}
again for $\theta\in\mathcal I_q$, $\delta\in\Delta_q$, and
$q=1,\ldots,Q$. 
Here the approximation uncertainty enters through the thrust-dependent control channel rather than the uncontrolled drift. For fixed \(u\), the model is affine in \(\delta\), so the robust conditions reduce to the four vertices of \(\Delta_q\).

\paragraph{Exact lifted representation}
Introducing $c$ and $s$ as
in~\eqref{eq:ct_auxiliary_variables} again removes $\theta$ from
the dynamics and the safety specification, giving the reduced
lifting map
\begin{equation}
\label{eq:pm_lifting_map}
    z=\phi(x)
    :=
    [p_x,p_y,v_x,v_y,c,s,\omega]^\top ,
\end{equation}
the auxiliary dynamics~\eqref{eq:ct_auxiliary_dynamics}, and the
exact polynomial model
$$
\dot z=
\begin{bmatrix}
    v_x\\
    v_y\\
    (u_1+u_2)s\\
    (u_1+u_2)c-g_0\\
    -s\omega\\
    c\omega\\
    u_1-u_2
\end{bmatrix},
$$
subject to the same algebraic
equality~\eqref{eq:ct_lifting_equality}.
\subsection{Post-Synthesis SOS Verification}
\label{app:post_synthesis_verification}
After synthesis, the resulting certificate polynomials are fixed
and independently re-verified using the representation
in~\eqref{eq:putinar}. For \gls{ldbb}, the safety and invariance
conditions in~\eqref{eq:sos_lifted_safety} and
\eqref{eq:sos_lifted_invariance} are checked on the exact lifted
polynomial model, while input admissibility is checked directly
on $\hat{\mathcal C}\cap\hat{\mathcal K}^{+}$.
For \gls{pwbb}, the corresponding checks use the certified
piecewise-polynomial model; invariance in~\eqref{eq:pwbb_C3}
is checked in every sector at every vertex of $\Delta_q$.
Because a low-degree relaxation may fail to certify a
valid inequality, the multiplier degree is progressively increased until verification succeeds. All reported certificates were successfully verified at degree \(12\).

\subsection{Nominal controller}
\label{app:nominal_control}
The closed-loop experiments use goal-tracking
controllers to generate \(u_{\mathrm{nom}}\). For \gls{CT},
position error is mapped to desired forward speed and heading
rate, which are tracked by proportional feedback. For \gls{PM},
a cascaded position--attitude controller maps position error to
desired acceleration and subsequently to attitude and thrust
commands with gravity compensation. Controller parameters are fixed across methods for each system and provided in the supplementary material.

% ==========================================================
\subsection{Evaluation Metrics}
\label{app:evaluation_metrics}
\paragraph{Certified coverage}
\label{app:cert_cov_met}
Let
$$
    \mathcal X_{\mathrm{test}}
    :=
    \bigl\{x^{(i)}\in\mathcal D\bigr\}_{i=1}^{N},
$$
where $N=10{,}000$ is the number of common physical test states sampled uniformly
from $\mathcal D$. Let $\{B_j\}_{j=1}^{n_B}$ denote the synthesized
barriers in the original coordinates; for the lifted methods,
$B_j(x):=\hat B_j(\phi(x))$ with $\phi$ given
by~\eqref{eq:ct_lifting_map} and~\eqref{eq:pm_lifting_map}. The
estimated certified coverage over $\mathcal X_{\mathrm{test}}$
is
$$
    \operatorname{Cov}
    :=
    \frac{1}{N}
    \sum_{i=1}^{N}
    \mathbbm{1}
    \left[
        \max_{j=1,\ldots,n_B}
        B_j\bigl(x^{(i)}\bigr)
        \geq 0
    \right].
$$
For \gls{pwbb} and \gls{ldbb}, $n_B=1$; for \gls{pwsb} and
\gls{ldsb}, $n_B$ is the number of barriers in the synthesized
bank.

\paragraph{Closed-loop metrics}
\label{app:closed_loop_metrics}
Consider a trajectory containing $N_{\mathrm r}$
zero-order-hold intervals with sampling instants
$t_k=kT_s$. Let $p_k:=p(t_k)$ for
$k=0,\ldots,N_{\mathrm r}$, and let $p_{\mathrm g}$ denote the
goal position. The accumulated intervention cost and goal-tracking cost are
$$
    J_{\mathrm{int},2}
    :=
    T_s
    \sum_{k=0}^{N_{\mathrm r}-1}
    \left\lVert
        u_k^\star-u_{\mathrm{nom},k}
    \right\rVert_2^2 ,
$$
$$
    J_g
    :=
    \frac{T_s}{2}
    \sum_{k=0}^{N_{\mathrm r}-1}
    \Bigl(
        \left\lVert p_k-p_{\mathrm g}\right\rVert_2
        +
        \left\lVert p_{k+1}-p_{\mathrm g}\right\rVert_2
    \Bigr).
$$

Finally, let $t_\ell^{\mathrm{RK}}$,
$\ell=0,\ldots,N_{\mathrm{RK}}$, denote the dense RK4
integration grid. Since the unsafe set
in~\eqref{eq:case_study_unsafe_set} is a disk of radius $r_o$
centered at the origin, the minimum obstacle clearance is
$$
    d_{\min}
    :=
    \min_{\ell=0,\ldots,N_{\mathrm{RK}}}
    \left(
        \bigl\lVert
            p(t_\ell^{\mathrm{RK}})
        \bigr\rVert_2
        -
        r_o
    \right).
$$

% \FloatBarrier
\bibliographystyle{IEEEtran}
\bibliography{Bibliography}

% \section*{ACKNOWLEDGMENT}
% We  thank  A,  B,  and  C.  This  work  was  supported  in  part  by  a grant from XYZ.
\end{document}